\documentclass{article}

\usepackage[accepted]{icml2023}

\usepackage{csquotes}
\usepackage{color}
\usepackage{colordvi}
\usepackage{amssymb}
\usepackage{amsthm}
\usepackage{colordvi}
\usepackage{float}

\usepackage{graphicx}
\usepackage{subcaption}
\usepackage{enumerate}
\usepackage{wrapfig}
\usepackage{enumitem}

\usepackage{bbm}
\usepackage{bigints}
\usepackage{soul}
\usepackage{pgffor}
\usepackage[colorlinks]{hyperref}  
\usepackage{xcolor}
\usepackage{wrapfig}
\usepackage{minitoc}
\usepackage[hide=false,setmargin=true,marginparwidth=0.7in]{marginalia}

\usepackage{bm}

\usepackage[%
    minnames=1,maxnames=99,maxcitenames=1,
    style=authoryear,
    sorting=ynt,
    sortcites=true,
    doi=false,url=false,
    giveninits=true,
    hyperref,natbib,backend=biber]{biblatex}
\renewbibmacro{in:}{%
  \ifentrytype{article}{}{\printtext{\bibstring{in}\intitlepunct}}}
\bibliography{sample}

\usepackage[capitalize]{cleveref}

\usepackage{mathtools}

\definecolor{niceblue}{rgb}{0.10, 0.14, 0.76} 

\definecolor{nicered}{rgb}{0.70, 0.0, 0.0} 

\AtBeginDocument{%
\hypersetup{
    citecolor=niceblue,
    linkcolor=red,   
    urlcolor=niceblue,
    linktoc=none}
}

\newcommand{\reals}{\mathbb{R}}

\newcommand{\ito}{It$\hat{o}$ }

\newcommand{\E}{\mathbb{E}}

\newcommand{\normal}{\mathcal{N}}
\newcommand{\ind}{\mathbbm{1}}
\newcommand{\iid}{\textit{iid}~}

\newcommand{\bigO}{\mathcal{O}}

\newtheorem{prop}{Proposition}
\newtheorem{thm}{Theorem}

\newtheorem{lemma}{Lemma}

\newtheorem{definition}{Definition}

\icmltitlerunning{Width and Depth commute in Residual Networks}

\begin{document}

\twocolumn[
\icmltitle{
Width and Depth Limits Commute in Residual Networks}




\begin{icmlauthorlist}
\icmlauthor{Soufiane Hayou}{nus} \hspace{1cm}
\icmlauthor{Greg Yang}{microsoft}
\end{icmlauthorlist}

\icmlaffiliation{nus}{Department of Mathematics, National University of Singapore}
\icmlaffiliation{microsoft}{Microsoft Research AI}

\icmlcorrespondingauthor{Soufiane Hayou}{hayou@nus.edu.sg}

\icmlkeywords{Machine Learning, ICML}

\vskip 0.3in
]
\printAffiliationsAndNotice{\noindent}

\begin{abstract}
We show that taking the width and depth to infinity in a deep neural network with skip connections, when branches are scaled by $1/\sqrt{depth}$, result in the same covariance structure no matter how that limit is taken. This explains why the standard infinite-width-then-depth approach provides practical insights even for networks with depth of the same order as width. We also demonstrate that the pre-activations, in this case, have Gaussian distributions which has direct applications in Bayesian deep learning. We conduct extensive simulations that show an excellent match with our theoretical findings.
\end{abstract}

\section{Introduction}

In recent years, deep neural networks have achieved remarkable success in a variety of tasks, such as image classification and natural language processing. However, the behavior of these networks in the limit of large depth and large width is still not fully understood. 

The success of large language and vision models have recently amplified an existing trend of research on neural network limits.
Two main limits are the large-width and the large-depth limits.
While the former by itself is now relatively well understood \citep{neal, samuel2017, lee_gaussian_process, yang_tensor3_2020, hayou2019impact}, the latter and the interaction between the two have not been studied as much.
In particular, a basic question is: do these two limits commute?
Recent literature suggests that, at initialization, in certain kinds of multi-layer perceptrons (MLPs) or residual neural networks (resnets), the depth and width limits do not commute; this would imply that in practice, such kinds of networks would behave quite differently depending on whether width is much larger than depth or the other way around.

However, in this paper, we show: to the contrary, at initialization, for a resnet with branches scaled the natural way so as to avoid blowing up the output,%
\footnote{This contrasts with \cite{li2022sde} whose non-commute result requires the branches to be large enough to blow up the network output in the case of standard resnet.}
the width and depth limits \emph{do commute}.
This justifies prior calculations that take the width limit first, then depth, to understand the behavior of deep residual networks, such as prior works in the signal propagation literature \citep{hayou21stable}.

In addition to the significance of the results, the mathematical novelty of this paper is the proof technique: we take the depth limit first (fixing width), then take the width limit, in contrast to the typical prior work which takes the limits in the opposite order. In the process, we prove a concentration of measure result for a kind of McKean-Vlasov process (Mean-Field games). Our results provide new insights into the behavior of deep neural networks and we discuss implications for the design and analysis of these networks.

The proofs of the theoretical results are provided in the appendix and referenced after each result. Empirical evaluations support our theoretical findings.

\section{Related Work}
The theoretical analysis of randomly initialized neural networks with an infinite number of parameters has yielded a wealth of interesting results, both theoretical and practical. A majority of this research has concentrated on examining the scenario in which the width of the network is taken to infinity while the depth is fixed. However, in recent years, there has been a growing interest in exploring the large depth limit of these networks. In this overview, we present a summary of existing results in this area, though it's not exhaustive. A more comprehensive literature review is provided in \cref{sec:comprehensive_lit_review}.

\subsection{Infinite-width limit}

The study of the infinite-width limit of neural network architectures has been a topic of significant research interest, yielding various theoretical and algorithmic innovations. These include initialization methods, such as the Edge of Chaos \citep{poole, samuel2017, yang2017meanfield, hayou2019impact}, and the selection of activation functions \citep{hayou2019impact, martens2021rapid, zhang2022deep, wolinski2022gaussian}, which have been shown to have practical benefits. In the realm of Bayesian analysis, the infinite-width limit presents an intriguing framework for Bayesian deep learning, as it is characterized by a Gaussian process prior. Several studies (e.g. \cite{neal, lee_gaussian_process, yang_tensor3_2020, matthews, hron20attention}) have investigated the weak limit of neural networks as the width increases towards infinity, and have demonstrated that the network's output converges to a distribution modeled by a Gaussian process. Bayesian inference utilizing this ``neural" Gaussian process has been explored in \citep{lee_gaussian_process, hayou21stable}. \footnote{It is worth mentioning that kernel methods such as NNGP and NTK significantly underperform properly tuned finite-width network trained using SGD, see \cite{yang2022efficient}.}
    
The Neural Tangent Kernel (NTK) is another interesting area of research where the infinite-width limit proves useful. In this limit, the NTK converges to a deterministic kernel, given appropriate parameterization. This limiting kernel is fixed at initialization and remains constant throughout the training process. The optimization and generalization characteristics of the NTK have been the subject of extensive study in the literature (see e.g. \cite{Liu2022connecting, arora2019finegrained}).

\subsection{Infinite-depth limit}
The infinite-depth limit of neural networks with random initialization is a less explored area compared to the study of the infinite-width limit. Existing research in this field can be categorized into three groups based on the approach and criteria used to consider the infinite-depth limit in relation to the width.

\emph{Infinite-width-then-depth limit.} In this case, the width of the neural network is taken to infinity first, followed by the depth. This is the infinite-depth limit of infinite-width neural networks. This limit has been extensively utilized to explore various aspects of neural networks, such as examining the neural covariance, deriving the Edge of Chaos initialization scheme (cited in \citep{samuel2017, poole, yang2017meanfield}), evaluating the impact of the activation function \citep{hayou2019impact, martens2021rapid},  and studying the behavior of the Neural Tangent Kernel (NTK) \citep{hayou_ntk, xiao2020disentangling}.

\emph{The joint infinite-width-and-depth limit.} In this case, the ratio of depth to width is fixed, and the width and depth are jointly taken to infinity. There are only a limited number of works that have investigated the joint width-depth limit. In \citep{li21loggaussian}, the authors showed that for a particular type of residual neural networks (ResNets), the network output exhibits a (scaled) log-normal behavior in this limit, which differs from the sequential limit in which the width is first taken to infinity followed by the depth, in which case the distribution of the network output is asymptotically normal (\citep{samuel2017, hayou2019impact}). Additionally, in \citep{li2022sde}, the authors examined the covariance kernel of a multi-layer perceptron (MLP) in the joint limit and proved that it weakly converges to the solution of a Stochastic Differential Equation (SDE). Other works have investigated this limit and found similar results \citep{noci2021precise, zavatone2021exact, Hanin2019product, hanin2022correlation}.
    
\emph{Infinite-depth limit of finite-width neural networks.}  In the previous limits, the width of the neural network was extended to infinity, either independently or in conjunction with the depth. However, it is natural to inquire about the behavior of networks in which the width is fixed, while the depth is increased towards infinity. In \cite{peluchetti2020resnetdiffusion}, it was shown that for a particular ResNet architecture, the pre-activations converge weakly to a diffusion process in the infinite-depth limit,  which follows from existing results in stochastic calculus on the convergence of Euler-Maruyama discretization schemes to continuous Stochastic Differential Equations. More recent work by \cite{hayou2022on} evaluated the impact of the activation function on the distribution of the pre-activation and characterized the distribution of the post-activation norms in this limit.

In this work, we are particularly interested in the case where both the width and depth are taken to infinity. 
\section{Setup and Definitions}
When analyzing the asymptotic behavior of randomly initialized neural networks, various notions of probabilistic convergence are employed, depending on the context. These notions are typically well-established definitions in probability theory. In this study, we particularly focus on two forms of convergence:
\vspace{-0.2cm}
\begin{itemize}
    \item Convergence in distribution (weak convergence): we show that the pre-activations converge weakly to a Gaussian distribution in the limit $\min(n, L) \to \infty$. We use the Wasserstein metric to quantify the convergence rate for the weak convergence.
    \item Convergence in $L_2$ (strong convergence): we show that the neural covariance\footnote{The neural covariance is a (linear) measure of similarity between the pre-activations for different inputs. We define this quantity in \cref{sec:mlps}.} converges to a deterministic limit that is characterized by a differential flow $q_t$ as $\min(n, L)$ approaches infinity.
\end{itemize}

\begin{definition}[Weak convergence]
Let $d \geq 1$. We say that a sequence of  $\reals^d$-valued random variables $(X_k)_{ k\geq 1}$ converges weakly to a random variable $Z$ if the cumulative distribution function of $X_k$ converges point-wise to that of $Z$.
\end{definition}

There are various metrics that can be utilized to measure the weak convergence rate. One commonly used metric is the Wasserstein metric.
\begin{definition}[Wasserstein distance $\mathcal{W}_1$]
Let $\mu$ and $\nu$ be two probability measures on $\reals^d$. The Wasserstein distance between $\mu$ and $\nu$ is defined by 
\begin{align*}
\mathcal{W}_1 &= \sup_{f \in \textup{Lip}_1} \left| \int f(x) (d \mu - d \nu) \right|\\
&=  \sup_{f \in \textup{Lip}_1} \left| \E_{\mu} f - \E_\nu f \right|,
\end{align*}
where $\textup{Lip}_1$ is the set of Lipschitz continuous functions from $\reals^d$ to $\reals$ with a Lipschitz constant $\leq 1$.
\end{definition}
In this work, we define \emph{strong} convergence to be the $L_2$ convergence as described in the following definition.
\begin{definition}[Strong convergence]
Let $d \geq 1$. We say that a sequence of  $\reals^d$-valued random variables $(X_k)_{ k\geq 1}$ converges in $L_2$ (or strongly) to a continuous random variable $Z$ if $\lim_{k \to \infty} \|X_k - Z\|_{L_2} =0$, where the $L_2$ is defined by $\|X\|_{L_2} = \left(\E[\|X\|^2] \right)^{1/2}$.
\end{definition}
Both of these forms of convergence are valuable when analyzing the behavior of neural networks with an infinite number of parameters. They facilitate the understanding of the network's asymptotic behavior which enables predictions about the finite-but-large width-and-depth regimes.

\section{Warmup: Depth and Width Generally Do Not Commute}\label{sec:mlps}
In this section, we present corollaries of previously established results that demonstrate that depth and width typically do not commute. The width and depth of the network are denoted by $n$ and $L$, respectively, and the input dimension is denoted by $d$. Let $d, n, L \geq 1$, and consider a simple MLP architecture given by the following: 
\begin{equation}\label{eq:mlp}
\begin{aligned}
Y_0(a) &= W_{in} a, \quad a \in \reals^d\\
Y_l(a) &= W_l \phi(Y_{l-1}(a)), \hspace{0.1cm} l \in [1:L],
\end{aligned}
\end{equation}
where $\phi: \reals \to \reals$ is the ReLU activation function, $W_{in} \in \reals^{n \times d}$, and $W_l \in \reals^{n \times n}$ is the weight matrix in the $l^{th}$ layer. We assume that the weights are randomly initialized with \iid Gaussian variables $W_l^{ij} \sim \normal(0, \frac{2}{n})$,\footnote{This is the standard He initialization which coincides with the Edge of Chaos initialization \citep{samuel2017}. This is the only choice of the variance that guarantees stability in both the large-width and the large-depth limits.} $W_{in}^{ij} \sim \normal(0, \frac{1}{d})$. For the sake of simplification, we only consider networks with no bias, and we omit the dependence of $Y_l$ on $n$ and $L$ in the notation. While the activation function is only defined for real numbers ($1$-dimensional), we will abuse the notation and write $\phi(z) = (\phi(z^1), \dots, \phi(z^k))$ for any $k$-dimensional vector $z = (z^1, \dots, z^k) \in \reals^k$ for any $k \geq 1$. We refer to the vectors $\{Y_l, l=0, \dots, L\}$ as \emph{pre-activations} and the vectors $\{\phi(Y_l), l=0, \dots, L\}$ as \emph{post-activations}.

\subsection{Distribution of the pre-activations in the limit $n, L \to \infty$}

It is well-established that in fixed-depth neural networks of any type, as the width $n$ approaches infinity, the pre-activations exhibit Gaussian behavior. This phenomenon was initially demonstrated for single-layer perceptrons by \citep{neal}, and has since been extended to include multiple-layer perceptrons (MLPs) and general neural architectures \citep{yang_tensor3_2020}. This behavior can be roughly attributed to the Central Limit Theorem (CLT) (although a formal proof require careful application of CLT for exchangeable random variables in the MLP case, as detailed in \cite{matthews}, or Law of Large Numbers and Gaussian conditioning trick in the general case \citep{yang2019tensor_i}). A question that the reader may have in this context is: \emph{Why is the Gaussian distribution of significance?}\\
One of the key implications of the Gaussian behavior of infinite-width neural networks is their equivalence to Gaussian processes. By utilizing existing methods of Gaussian process regression, this equivalence facilitates the application of exact Bayesian inference to infinite-width neural networks, referred to as the neural network Gaussian process (NNGP, \cite{lee_gaussian_process}). The Gaussian behavior also provides an interesting framework to study signal propagation in deep neural networks; since a Gaussian distribution is fully characterized by its mean and covariance structure, understanding these quantities is sufficient to capture what happens inside the network at initialization. 

When the depth $L$ is also taken to infinity, different behaviors may emerge. Specifically, in the case of the MLP architecture \eqref{eq:mlp}, if a fixed layer index $l < L$ is considered and the behavior of $Y_l$ is examined as $n$ and $L$ approach infinity, $Y_l$ will exhibit the same limiting behavior as in the case of $n \to \infty$ and the depth is fixed. 
Some simple intuitive calculations indicate that it is only meaningful to study the limiting behavior of layers where the layer index is proportional to the depth $L$ (and not proportional to $L^\alpha$ for any $\alpha < 1$).%
\footnote{Indeed, the $(\frac 1 n)$-scaled Gram matrix of $\{Y_l(a): a \in \mathbb{R}^d\}$ fluctuates with size $\tilde \Theta(1/\sqrt{n})$ around its $n\to \infty$ limit for any fixed $l$.
This fluctuation is asymptotically independent across every layer, so the accumulated fluctuation at layer $l=L^\alpha$ is $\tilde \Theta(L^{\alpha/2}/\sqrt{n})$.
This is $\tilde \Theta(1)$ iff $\alpha = 1$.}
In this case, the quantity of interest is $Y_{\lfloor t L \rfloor}$ for some $t \in [0,1]$. Varying $t$ between $0$ and $1$ encompasses all layer indices, even in the infinite-depth limit.

Let us now state some corollaries of existing results. The following is a trivial result from existing literature (see e.g. \cite{matthews}) that characterizes the distribution of the pre-activations in the limit $n \to \infty$ then $L \to \infty$.

\begin{prop}[Infinite-width-then-depth]\label{prop:mlp_width_then_depth_gaussian}
Consider the MLP architecture given by \cref{eq:mlp} and let $a \in \reals^d$ such that $a \neq 0$. Then, in the limit ``$n\to \infty$, then $L \to \infty$'', $Y^1_{L}(a)$\footnote{$Y^1_{L}(a)$ refers to the first neuron in the last layer.} converges weakly to a Gaussian distribution. 
\end{prop}

When the width and depth of a neural network both tend towards infinity, the limiting behavior can vary depending on the relative rates at which the width and depth increase. Specifically, if the width and depth both approach infinity while the ratio of width to depth remains constant, the distribution of the pre-activations in the last layer is not Gaussian. This is a corollary of a more general result established by \citep{li21loggaussian} (the case when $\alpha = 0$) under certain conditions and assumptions, which was also verified through empirical evidence. We omit here the rigorous statement of the result and only illustrate this behaviour with simulations.

\begin{figure}
    \centering
    \includegraphics[width=\linewidth]{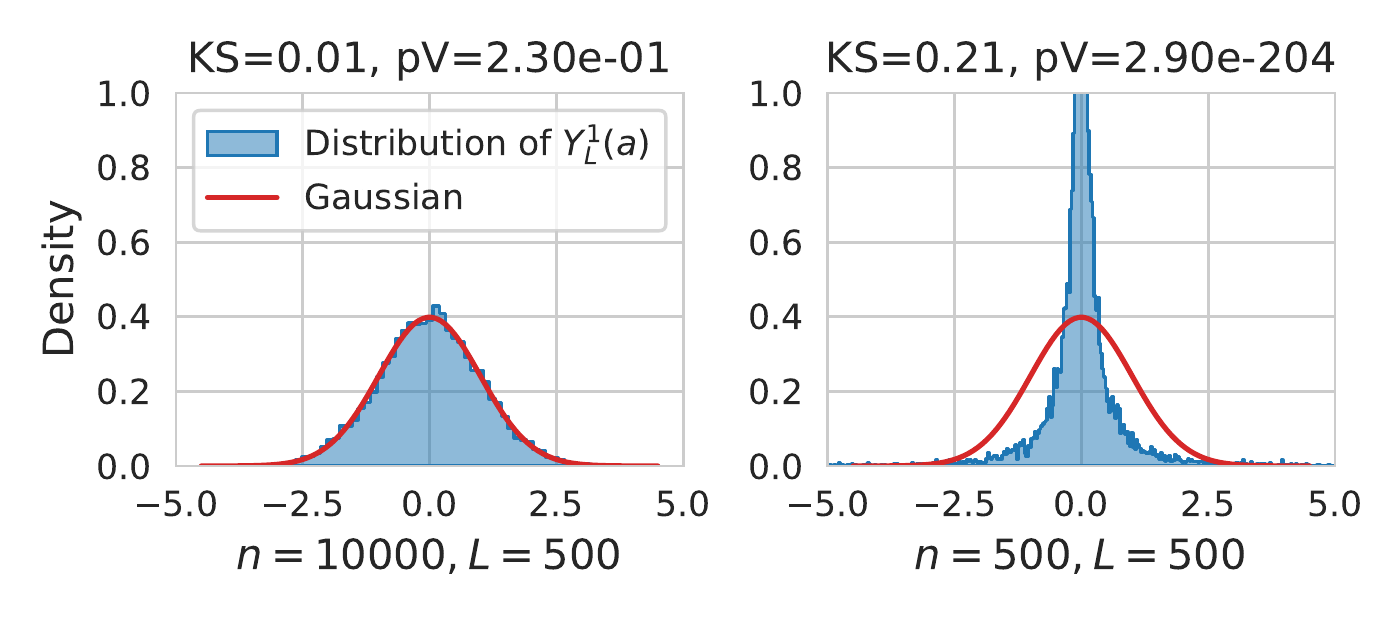}
    \caption{Histogram of $Y^1_L(a)$ for an MLP \cref{eq:mlp} with $(n, L) \in \{(10000, 500), (500, 500)\}$, $d=30$,  and $a = \sqrt{d} \frac{u}{\|u\|}$ and $u \in \reals^d$ has all coordinates randomly sampled from the unifrom distribution $\mathcal{U}([0,1])$. The histogram is based on $N=10^4$ simulations. The red dashed line represents the theoretical distribution (Gaussian) predicted in \cref{prop:mlp_width_then_depth_gaussian}. We also perform a Kolmogorov-Smirnov normality test and report the KS statistic and the p-value.}
    \label{fig:non_gaussian_mlp}
    \vspace{-0.3cm}
\end{figure}

Empirical evidence supports the existence of this difference in the limiting behavior of the distribution. As shown in \cref{fig:non_gaussian_mlp}, the distribution of $Y^1_L(a)$ is observed to be (nearly) Gaussian when the width is significantly greater than the depth, as evidenced by a small KS statistic. However, when the width is of the same magnitude as the depth, the distribution exhibits heavy tails. This can be seen by comparing the distribution for the settings $(n, L) \in {(10000, 500), (500, 500)}$.

\subsection{Neural covariance/correlation}
In the literature on signal propagation, there is a significant interest in understanding the covariance/correlation structure of neural networks. Specifically, researchers have sought to understand the covariance of the pre-activation vectors $Y_{\lfloor t L \rfloor}(a)$ and $Y_{\lfloor t L \rfloor}(b)$ (often called the neural covariance) for two different inputs $a, b \in \reals^d$. A natural question in this context is: \emph{Why do we study the covariance structure?}

It is well-established that even for properly initialized multi-layer perceptrons (MLPs), the network outputs $Y_L(a)$ and $Y_L(b)$ become perfectly correlated (correlation=1) in the limit of ``$n \to \infty$, \emph{then} $L \to \infty$'' \citep{samuel2017, poole, hayou2019impact, yang2019fine}.  This can lead to unstable behavior of the gradients and make the model untrainable as the depth increases and also results in the inputs being non-separable by the network\footnote{To see this, assume that the inputs are normalized. In this case, the correlation between the pre-activations of the last layer for two different inputs converges to 1. This implies that as the depth grows, the network output becomes similar for all inputs, and the network no longer separates the data. This is problematic for the first step of gradient descent as it implies that the information from the data is (almost) unused in the first gradient update.}. To address this issue, several techniques involving targeted modifications of the activation function have been proposed \citep{martens2021rapid, zhang2022deep}. In the case of ResNets, the correlation still converges to 1, but at a polynomial rate \citep{yang2017meanfield}. A solution to this problem has been proposed by introducing well-chosen scaling factors in the residual branches, resulting in a correlation kernel that does not converge to 1 \citep{hayou21stable}. This analysis was carried in the limit ``$n \to \infty$, then, $L \to \infty$''. In the case of the joint limit $n,L \to \infty$ with $n/L$ fixed, it has been shown that the covariance/correlation between $Y_{\lfloor t L \rfloor}(a)$ and $Y_{\lfloor t L \rfloor}(b)$ becomes similar to that of a Markov chain that incorporates random terms. However, the correlation still converges to one in this limit.
\begin{prop}[Correlation, \citep{hayou2019impact, li2022sde}]\label{prop:covariance_mlp}
Consider the MLP architecture given by \cref{eq:mlp} and let $a, b \in \reals^d$ such that $a, b \neq 0$. Then, in the limit ``$n\to \infty$, then $L \to \infty$'' or the the joint limit ``$n ,L \to \infty$, $L/n$ fixed'', the correlation $\frac{\langle Y_{L}(a), Y_{L}(b) \rangle}{ \|Y_{L}(a)\| \|Y_{L}(b)\|}$ converges\footnote{Note that weak convergence to a constant implies also convergence in probability.} weakly to 1.
\end{prop}


The convergence of the correlation to 1 in the infinite depth limit of a neural network poses a significant issue, as it indicates that the network loses all of the covariance structure from the inputs as the depth increases. This results in 
 degenerate gradients (see e.g. \citep{samuel2017}), rendering the network untrainable. To address this problem in MLPs, various studies have proposed the use of depth-dependent shaped ReLU activations, which prevent the correlation from converging to 1 and exhibit stochastic differential equation (SDE) behavior. As a result, the correlation of the last layer does not converge to a deterministic value in this case.
\begin{prop}[Correlation SDE, Corollary of Thm 3.2 in \cite{li2022sde}]\label{prop:covariance_shaped_mlp}
Consider the MLP architecture given by \cref{eq:mlp} with the following activation function $\phi_L(z) = z + \frac{1}{\sqrt{L}} \phi(z) $ (a modified ReLU). Let $a, b \in \reals^d$ such that $a, b \neq 0$. Then, in the joint limit ``$n ,L \to \infty$, $L/n$ fixed'', the correlation $\frac{\langle Y_{L}(a), Y_{L}(b) \rangle}{ \|Y_L(a)\| \|Y_L(b)\|}$  converges weakly to a nondeterministic random variable.\footnote{In \cite{li2022sde}, the authors show that the correlation of $\frac{\langle \phi_L(Y_{L}(a)), \phi_L(Y_{L}(b)) \rangle}{\sqrt{ \|\phi_L(Y_{L}(a))\|} \sqrt{ \|\phi_L(Y_{L}(b))\|}}$ converges to a random variable in the joint limit. Since $\phi_L$ converges to the identity function in this limit, simple calculations show that the correlation between the pre-activations $\frac{\langle Y_{L}(a), Y_{L}(b) \rangle}{\|Y_{L}(a)\| \|Y_{L}(b)\|}$ is also random in this limit.}
\end{prop}
The joint limit, therefore, yields non-deterministic behaviour of the covariance structure. It is easy to check that even with shaped ReLU as in \cref{prop:covariance_shaped_mlp}, taking the width to infinity first, then depth, the result is a deterministic covariance structure. The main takeaway from this section is the following:
\paragraph{Summary.} \emph{With MLPs (\cref{eq:mlp}), the width and depth limits do not commute in the sense that the behaviour of the distribution of the pre-activations and the covariance structure might differ depending on how the limit is taken.}

With the background information provided above, we are now able to present our findings. In contrast to MLPs, our next section demonstrates that the limits of width and depth for ResNet architectures commute.
\section{Main results: Width and Depth Commute in ResNets}
We use the same notation as in the MLP case. Let $d, n, L \geq 1$, and consider the following ResNet architecture of width $n$ and depth $L$
\begin{equation}\label{eq:resnet}
\begin{aligned}
Y_0(a) &= W_{in} a, \quad a \in \reals^d\\
Y_l(a) &= Y_{l-1}(a) + \frac{1}{\sqrt{L}} W_l \phi(Y_{l-1}(a)), \hspace{0.1cm} l \in [1:L],
\end{aligned}
\end{equation}
where $\phi: \reals \to \reals$ is the ReLU activation function. We assume that the weights are randomly initialized with \iid Gaussian variables $W_l^{ij} \sim \normal(0, \frac{1}{n})$, $W_{in}^{ij} \sim \normal(0, \frac{1}{d})$. For the sake of simplification, we only consider networks with no bias, and we omit the dependence of $Y_l$ on $n$ and $L$ in the notation. 

The $1/\sqrt{L}$ scaling in \cref{eq:resnet}is not chosen arbitrarily. It has been demonstrated that this specific scaling serves to stabilize the norm of $Y_l$ and the gradient norms in the asymptotic limit of large depth (e.g. \cite{hayou2022on, hayou21stable, marion2022scaling}).\footnote{A scaling of the form $L^{-\alpha}$ where $\alpha < 1/2$ yields exploding pre-activations, while a more aggressive scaling where $\alpha>1/2$ yields trivial limiting covariance (identity covariance).}

\subsection{Distribution of the pre-activations in the limit $n, L \to \infty$}

It turns out that for the ResNet architecture given by \eqref{eq:resnet}, the limiting distribution of the pre-activations $Y_{\lfloor t L\rfloor}$ is a zero-mean Gaussian distribution, with an analytic variance term, regardless of how the depth $L$ and width $n$ approach infinity, as long as $\min(n, L) \to \infty$. This is demonstrated in the following result, where an upper bound on the Wasserstein distance between the distribution of the neuron $Y^1_{\lfloor t L\rfloor}$ (the first coordinate of the pre-activations $Y_{\lfloor t L\rfloor}$)\footnote{Notice that the coordinate of the pre-activations are identically distributed (but not necessarily independent).} and that of a zero-mean Gaussian random variable is provided.

\begin{thm}[Convergence of the pre-activations]\label{thm:gaussian_limit}
Let $a \in \reals^d$ such that $a \neq 0$. For $t \in [0,1]$, the random variable $(Y_{\lfloor t L\rfloor}(a))_{L \geq 1}$ converges weakly to a Gaussian random variable with law $\normal(0, v(t,a))$ in the limit of $\min(n,L) \to \infty$, where $v(t,a) = d^{-1} \|a\|^2 \exp(t/2)$.
Moreover, we have the following convergence rate
$$
\sup_{t \in [0,1]} \mathcal{W}_1(\mu^t_{n,L}(a), \mu^t_{\infty, \infty}(a)) \leq C \left(\frac{1}{\sqrt{n}} + \frac{1}{\sqrt{L}} \right)
$$
where $\mu^t_{n,L}(a)$ is the distribution of $Y^1_{\lfloor t L\rfloor}(a)$, $\mu^t_{\infty, \infty}(a)$ is the distribution $\normal(0, v(t,a))$, and $C$ is a constant that depends only on $\|a\|$ and $d$.

Moreover, for two different $i,j \in [n]$, the neurons $Y^i_{\lfloor t L\rfloor}(a)$ and $Y^j_{\lfloor t L\rfloor}(a)$ become independent in the limit $\min(n,L) \to \infty$.
\end{thm}
The proof of \cref{thm:gaussian_limit} is provided in \cref{sec:proof_gaussian_limit}. It relies on two technical results: 1) Width-uniform convergence rate of the finite-width neural networks to an infinite-depth SDE.\footnote{By width-uniform, we refer to bounds with constants that do not depend on the width $n$.} 2) A new result on the convergence of particles to a mean field process. Both results are new. More details are provided in the Appendix.

\cref{thm:gaussian_limit} suggests that the distribution of the pre-activations becomes similar to a Gaussian distribution as $\min(n,L) \to \infty$ regardless of how $n$ and $L$ go to infinity. Note that the limiting distribution is the same as the one reported in \citep{hayou2022on} where the author considered the limit ``$n \to \infty$, \emph{then} $L \to \infty$''. Our result generalizes these findings and establishes the universality of the Gaussian behaviour as long as $n \to \infty$ and $L\to \infty$. We validate these theoretical predictions in \cref{sec:experiments}. An important consequence of the Gaussian behaviour is that the residual network can be seen as a Gaussian process in this limit with a well-specified kernel function (see next section). Leveraging this result to perform Bayesian inference with infinite-width-and-depth networks can be an interesting direction for future work.

\subsection{Neural covariance}

Unlike the covariance structure in MLPs which exhibits different limiting behaviors depending on how the width and depth limits are taken, we show in the next result that for the ResNet architecture given by \eqref{eq:resnet}, the neural covariance converges strongly to a deterministic kernel, which is given by the solution of a differential flow, in the limit $\min(n,L) \to \infty$ regardless of the relative rate at which $n$ and $L$ tend to infinity.

\begin{thm}[Neural covariance]\label{thm:covariance}
Let $a, b \in \reals^d$ such that $a, b \neq 0$ and $a \neq b$. Define the neural covariance kernel $\hat{q}_t(a,b) = \frac{\langle Y_{\lfloor t L\rfloor}(a), Y_{\lfloor t L\rfloor}(b) \rangle}{n}$. Then, we have the following 
$$
\sup_{t \in [0,1]} \left\| \hat{q}_t(a,b) - q_t(a,b)\right\|_{L_2} \leq C \left(\frac{1}{\sqrt{n}} + \frac{1}{\sqrt{L}} \right)
$$
where $C$ is a constant that depends only on $\|a\|$, $\|b\|$, and $d$, and $q_t(a,b)$ is the solution of the following differential flow
\begin{equation}
\begin{aligned}
\begin{cases}
\frac{d q_t(a,b)}{dt} &= \frac{1}{2} \frac{f(c_t(a,b))}{c_t(a,b)} q_t(a,b),\\
c_t(a,b) &= \frac{q_t(a,b)}{ \sqrt{q_t(a,a)} \sqrt{q_t(b,b)}},\\
q_0(a,b) &= \frac{\langle a, b \rangle}{d},
\end{cases}
\end{aligned}
\end{equation}
where the function $f: [-1,1] \to [-1,1]$ is given by 
$$
f(z) = \frac{1}{\pi} ( z \arcsin(z) + \sqrt{1 - z^2}) + \frac{1}{2}z.
$$
\end{thm}

The proof of \cref{thm:covariance} is provided in \cref{sec:proof_covariance}. The result of \cref{thm:covariance} unifies previous approaches to understanding the covariance structure in large width and depth ResNets. Perhaps the most important consequence of our result is that it implies that all previous results that considered the limit $n\to\infty$, then $L \to \infty$, in order to understand the covariance structure in ResNets still hold for ResNets where the depth is of the same order as the width and both are large. This is specific for ResNet and does not hold for instance for MLPs where the joint-limit yields different asymptotic behaviors (see \cref{sec:mlps}). Notice that the limiting covariance kernel $q_t$ is the same kernel found in \citep{hayou21stable} in the limit $n \to \infty,$ then $L \to \infty$.\footnote{In \cite{hayou21stable}, the authors showed that the kernel $q_t$ is universal, meaning the network output is rich enough that we can approximate any continuous function on a compact set with features from this kernel.} It is also worth noting that constant $C$ can be chosen independent of $\|a\|$ and $\|b\|$ provided that the inputs belong to a compact set that does not contain $0$. The result of \cref{thm:covariance} can also be expressed in terms of the correlation. We demonstrate this in the next theorem.
\begin{thm}[Neural correlation]\label{thm:correlation}
Under the same conditions of \cref{thm:covariance}, we have the following 
$$
\sup_{t \in [0,1]} \left\| \hat{c}_t(a,b) - c_t(a,b)\right\|_{L_2} \leq C' \left(\frac{1}{\sqrt{n}} + \frac{1}{\sqrt{L}} \right)
$$
where $C'$ is a constant that depends only on $\|a\|$, $\|b\|$, and $d$, and $\hat{c}_t(a,b) = \frac{\langle Y_{\lfloor t L\rfloor}(a), Y_{\lfloor t L\rfloor}(b) \rangle}{\|Y_{\lfloor t L\rfloor}(a)\| \|Y_{\lfloor t L\rfloor}(b)\|}$ is the neural correlation kernel, and $c_t(a,b)$ is defined in \cref{thm:covariance}.
\end{thm}

The proof of \cref{thm:correlation} relies on using a concentration inequality to control the inverse variance term, and conclude by using the bound in \cref{thm:covariance}. We refer the reader to the Appendix for more details.

The differential flow satisfied by the kernel function $q_t$ can actually be simplified and expressed as an ordinary differential equation (ODE). We show this in the next lemma.

\begin{lemma}\label{lemma:kernel_ode}
Let $z = (a,b) \in \reals^d \times \reals^d$. The function $q_t$ in \cref{thm:covariance} is the solution of the following ODE:
$$
\frac{d q_t(z)}{dt} = \frac{\exp(t/2)}{2} \xi(z) f\left(\xi(z)^{-1} \exp(-t/2) q_t(z)\right),
$$
where $\xi(z) = \frac{\|a\| \, \|b\|}{d}$, and $f$ is defined in \cref{thm:covariance}.
\end{lemma}
\begin{proof}
The proof is straightforward by noticing that $f(1)=1$. With this we get $\frac{d q_t(a,a)}{dt} = \frac{1}{2} q_t(a,a)$ which yields $q_t(a,a) = q_0(a,a) \exp(t/2) = d^{-1} \|a\|^2 \exp(t/2)$. The same holds for $b$, which concludes the proof.
\end{proof}

\cref{lemma:kernel_ode} will prove useful in the experiments section when we will have to approximate the solution $q_t$ using ODE solvers.

\section{Experiments and Practical Implications}\label{sec:experiments}
In this section, we validate our theoretical results with extensive simulations on large width and depth residual neural networks of the form \cref{eq:resnet}. 
\subsection{Gaussian behavior and independence of neurons}
\cref{thm:gaussian_limit} predicts that in the large depth and width limit, the neurons (pre-activations) converge weakly to a Gaussian distribution. To empirically validate this finding, we show in \cref{fig:hist_with_ks} the histograms of the first neuron in the last layer ($t=1$ in \cref{thm:gaussian_limit}) for a randomly chosen input $a$ and $n, L \in \{5, 50, 500\}$. We also perform a Kolmogorov-Smirnov normality test and report the statistic ($KS$) and the p-value. As can be seen in \cref{fig:hist_with_ks}, the histograms appear to fit the theoretical Gaussian distribution more closely as width and depth increase. Additionally, the KS statistic decreases as the width and depth increase. For smaller widths, the p-values are extremely small indicating a non-Gaussian behavior. This is expected as the Gaussian behavior arises primarily due to the average behavior when the width increases. The depth also plays a role in the goodness of fit, as can be seen for the pair $(n,L) = (500,50)$ and $(n,L) = (500,500)$ where the latter shows a better fit in terms of the KS statistic which measures the distance between the empirical cumulative distribution function and the theoretical one. Notice also the contrast with the previously reported case of MLP (\cref{fig:non_gaussian_mlp}) where the the distribution of the neurons in the last layer is heavy-tailed.
\begin{figure}
    \centering
    \includegraphics[width=1.1\linewidth]{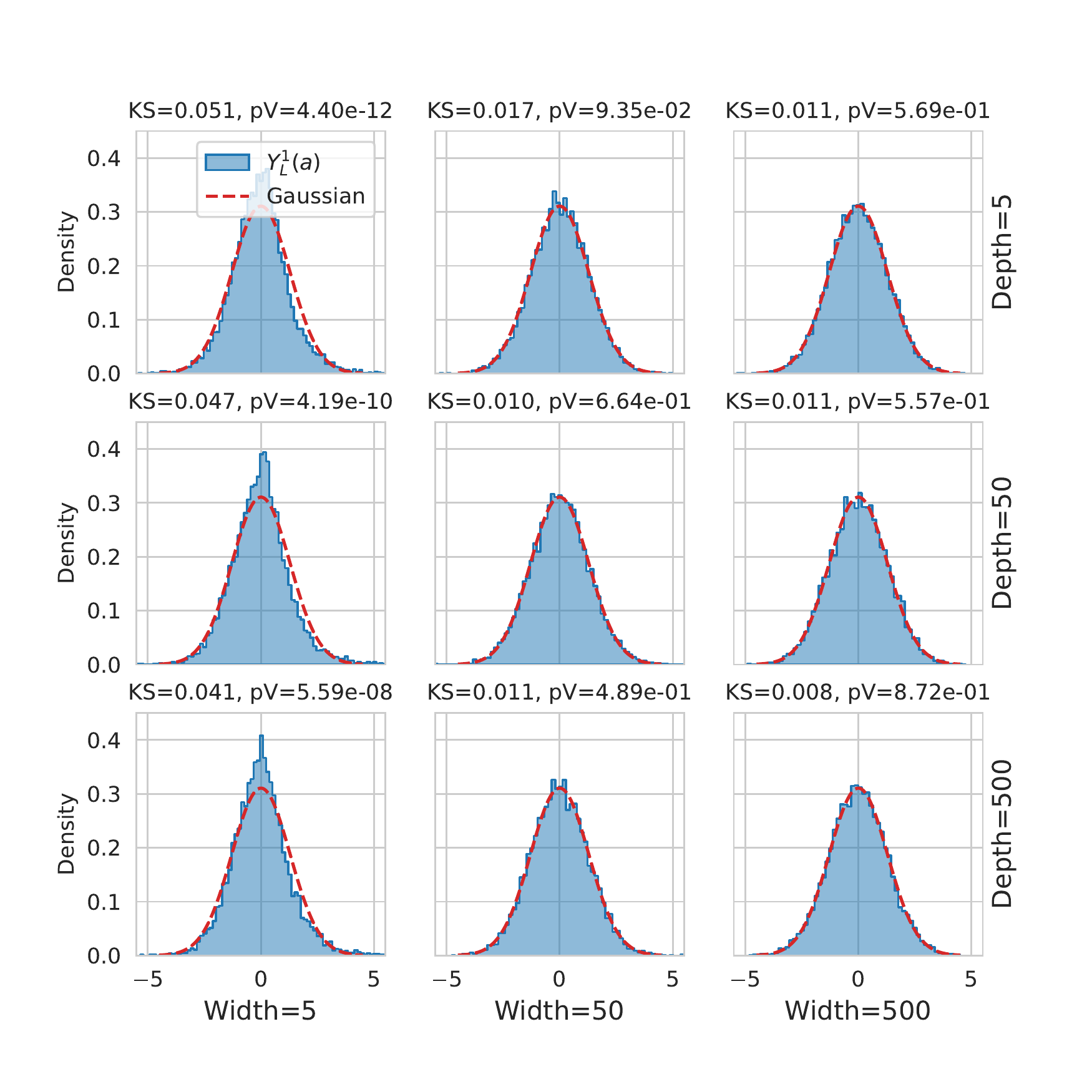}
    \caption{Histogram of $Y^1_L(a)$ for ResNet \cref{eq:resnet} with $n, L \in \{5, 50, 500\}$, $d=30$,  and $a = \sqrt{d} \frac{u}{\|u\|}$ and $u \in \reals^d$ has all coordinates randomly sampled from the unfirom distribution $\mathcal{U}([0,1])$. The histogram is based on $N=10^4$ simulations. The red dashed line represents the theoretical distribution (Gaussian) predicted in \cref{thm:gaussian_limit}. We also peform a Kolmogorov-Smirnov normality test and report the KS statistic and the p-value.}
    \label{fig:hist_with_ks}
    \vspace{-0.2cm}
\end{figure}

Another theoretical prediction of \cref{thm:gaussian_limit} is the independence of the neurons $(Y^i_{\lfloor tL \rfloor})_{1 \leq i \leq L}$. To validate this prediction, we show in figure \cref{fig:pairplot_500x500} the pair-wise joint distributions of 3 randomly chosen neurons in the last layer ($t = 1$). We also perform a kernel density estimation (KDE) using the Gaussian kernel and illustrate the result on top of the histograms. The joint distributions show an excellent match with an isotropic 2-dimensional Gaussian distribution which indicates independence of the neurons.

In \cref{fig:joyplots}, we investigate the distribution of the first neuron in each layer in a ResNet/MLP of width $n=500$ and depth $L=500$. For the ResNet architecture, the distribution is relatively similar across layers which is expected since \cref{thm:gaussian_limit} predicts a Gaussian limit with a standard deviation that differs only by a factor of $e^{1/4} \approx 1.28$ between the first and the last layers. In MLPs, the distribution varies across layers with the neurons in the last layers displaying heavy-tailed shapes, which agrees with \cref{fig:non_gaussian_mlp}.

\begin{figure}
    \centering
    \includegraphics[width=\linewidth]{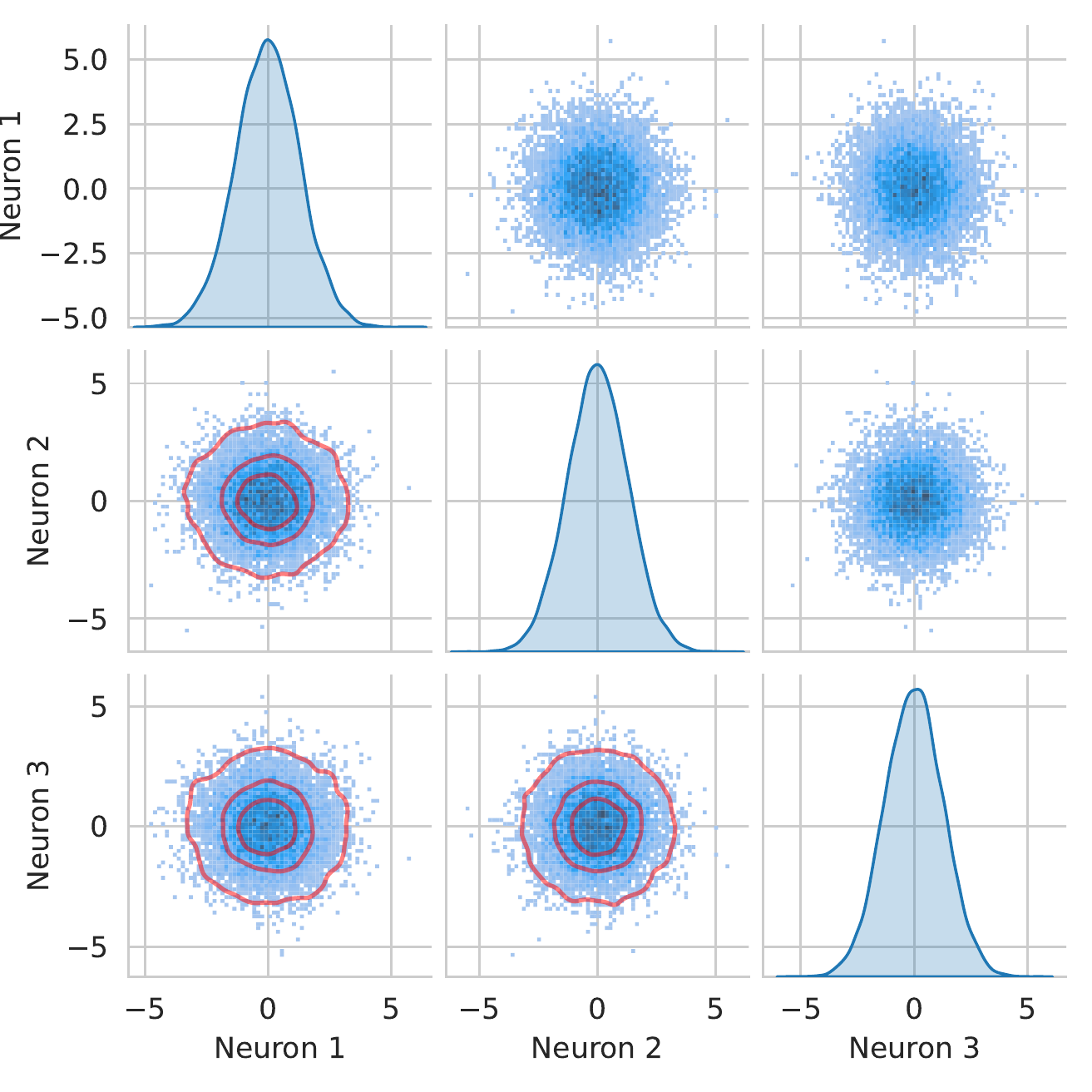}
    \caption{Joint distributions of $(Y^i_L(a), Y^j_L(a))$ for ResNet \cref{eq:resnet} with $n, L = 500$, $d=30$, $i,j \in \{i_1, i_2, i_3\}$ where $i_1,i_2, i_3$ are randomly sampled from $\{1, \dots, n\}$, and $a = \sqrt{d} \frac{u}{\|u\|}$ and $u \in \reals^d$ has all coordinates randomly sampled from the uniform distribution $\mathcal{U}([0,1])$. The histograms are based on $N=10^4$ simulations. The red curves represent an isotropic two-dimensional Gaussian distribution (i.e. independent coordinates).}
    \label{fig:pairplot_500x500}
    \vspace{-0.2cm}
\end{figure}

\begin{figure}
     \centering
     \begin{subfigure}[b]{0.2\textwidth}
         \centering
         \includegraphics[width=\textwidth]{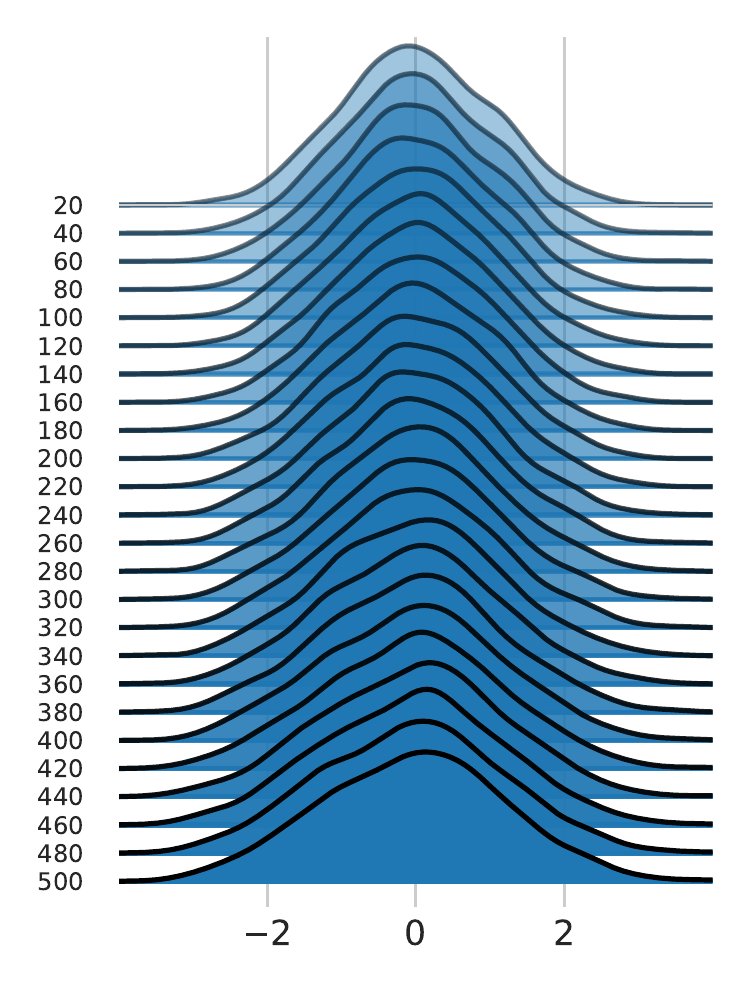}
         \caption{ResNet}
     \end{subfigure}
     \hfill
     \begin{subfigure}[b]{0.2\textwidth}
         \centering
         \includegraphics[width=\textwidth]{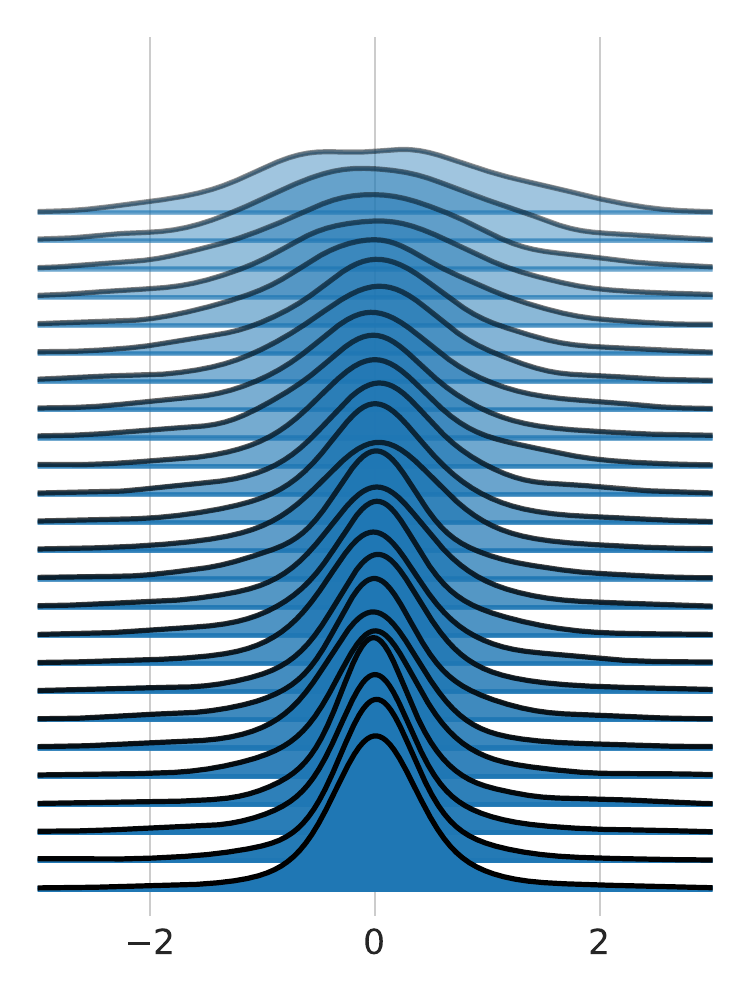}
         \caption{MLP}
     \end{subfigure}
        \caption{Densities (approximated by Kernel Density Estimation) of the first neuron $Y^1_{l}(a)$ for $l \in \{ 20 k, k=1, \dots,25\}$ for a ResNet \cref{eq:resnet} and an MLP \cref{eq:mlp} with $(n, L) = (500, 500)$. The input $a$ is randomly sampled and normalized in the same way as in \cref{fig:hist_with_ks}.}
        \label{fig:joyplots}
\end{figure}

\subsection{Convergence of neural covariance}
\cref{thm:covariance} predicts that the covariance $\hat{q}_t(a,b)$ for two inputs $a, b$ converges in $L_2$ norm to $q_t$ in the limit $\min(n,L) \to \infty$. In \cref{fig:covariance}, we compare the empirical covariance $\hat{q}_t$ with the theoretical prediction $q_t$ for $(n,L) \in \{5, 50, 500, 5000\}$. The empirical $L_2$ error is also reported. As the width increases, we observe a good match with the theory. The role of the depth is less visually noticeable, but for instance, with width $n = 5000$, we can see that the $L_2$ error is smaller with depth $L=5000$ as compared to depth $L=5$ (see \cref{sec:discussion_width_depth} for a more in-depth discussion of the role of width and depth). The theoretical prediction $q_t$ is approximated with a PDE solver (RK45 method, \cite{Fehlberg1968ClassicalFS}) for $t \in [0,1]$ with a discretization step $\Delta t = $1e-4.

\subsection{Role of width and depth}\label{sec:discussion_width_depth}
From \cref{fig:hist_with_ks} and \cref{fig:covariance}, it appears that the role of the width is more important than that of the depth in the convergence to the limiting values. In this section, we provide an intuitive explanation as to why that happens. First of all, recall that in both figures, the impact of depth is less noticeable but reflected in some measures (KS statistic in \cref{fig:hist_with_ks}, and $L_2$ error in \cref{fig:covariance}). The bounds in \cref{thm:gaussian_limit} and \cref{thm:covariance} are of the form $C \left(\frac{1}{\sqrt{n}} + \frac{1}{\sqrt{L}} \right)$ for some constant $C$. This bound is sufficient to conclude on the convergence rate but it is not optimal in terms of the constants. We conjecture that a `better' bound of the form $\frac{C_1}{\sqrt{n}} + \frac{C_2}{\sqrt{L}}$ can be obtained where the constant $C_2$ is much smaller than $C_1$, which would explain why the depth has less impact on the bound. To give the reader an intuition of why this should be the case, let us look at the case where the width is much larger than the depth, for instance $n=500$ and $L \in \{5,50\}$ (see \cref{fig:hist_with_ks}). Since $n \gg L$, then we are essentially in the regime where the $n$ goes to infinity first. In this case, the impact of depth is limited to how far the finite-depth variance is from infinite-depth one $v(t,a)$ (see \cref{thm:gaussian_limit}). For an input satisfying $\|a\|^2 = d$, simple calculations yield that the infinite-width finite-depth $L$ variance of the neurons in the last layer is given by $ \sigma_L = (1 + \frac{1}{2L})^L$.\footnote{See e.g.  \cite{hayou21stable}.} For $L=5$, $\sigma_5 \approx 1.61$ and for $L=50$, we have $\sigma_{50} \approx 1.644$. This is very close to the infinite-depth variance given by $v(1,a) = e^{1/2} \approx 1.648$. Hence, even for small depths, the finite-depth variance is close to the infinite-depth variance. Similar analysis can be carried for the covariance as well.

\section{Conclusion and Limitations}

In this paper, we have shown that, at initialization, in the most natural scaling of branches, the large-depth and large-width limits of a residual neural network (resnet) commute. We used a novel proof technique and proved a concentration of measure result for a kind of McKean-Vlasov process. Our results justify the calculations in prior works analyzing deep and wide neural networks that take the width limit first then depth.
\begin{figure}
    \centering
    \includegraphics[width=1.1\linewidth]{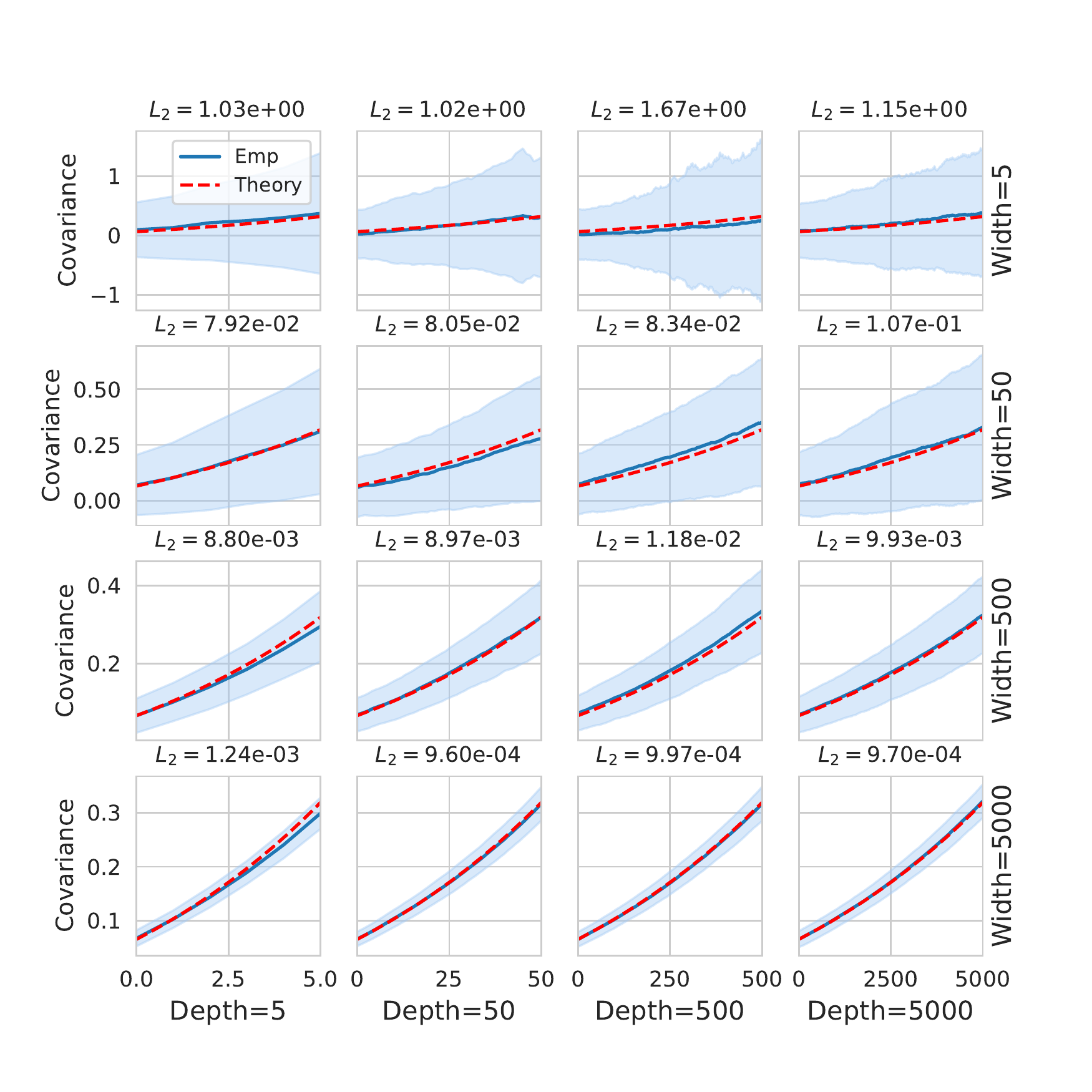}
    \caption{The blue curve represents the average covariance $\hat{q}_t(a,b)$ for ResNet \cref{eq:resnet} with $n, L \in \{5, 50, 500, 5000\}$, $d=30$, and $a$ and $b$ are sampled following the same rule as in \cref{fig:hist_with_ks}. The average is calculated based on $N=100$ simulations. The shaded blue area represents 1 standard deviation of the observations. The red dashed line represents the theoretical covariance $q_t(a,b)$ predicted in \cref{thm:covariance}. The empirical $L_2$ error is reported as well.}
    \label{fig:covariance}
\end{figure}
However, our technique cannot say anything about what happens when the network starts training.
Potentially, different behaviors can occur depending on how the learning rate is chosen as a function of width and depth.
Because of the correlations between weights induced by training, such an analysis would likely require far more mathematical machinery than presented here, e.g., Tensor Programs \citep{yang2019scaling,yang2019tensor_i,yang2021tensor_iv,yangTP2b,yangTP2,yangTP5}.

\printbibliography

\newpage
\appendix
\onecolumn




\newpage
\section{A more comprehensive literature review}\label{sec:comprehensive_lit_review}
Theoretical analysis of randomly initialized neural networks with an infinite number of parameters has yielded a wealth of interesting results, both theoretical and practical. Most of the research in this area has focused on the case where the depth of the network is fixed and the width is taken to infinity. However, in recent years, motivated by empirical observations, there has been an increased interest in studying the large depth limit of these networks. We provide here a non-exhaustive summary of existing results of these limits.

\subsection{Infinite-width limit}
The infinite-width limit of neural network architectures has been extensively studied in the literature and has led to many interesting theoretical and algorithmic innovations. We summarize these results below. 
\begin{itemize}[leftmargin=*]
    \item \emph{Initialization schemes}: the infinite-width limit of different neural architectures has been  extensively studied in the literature. In particular, for multi-layer perceptrons (MLP), a new initialization scheme that stabilizes forward and backward propagation (in the infinite-width limit) was derived in \citep{poole, samuel2017}. This initialization scheme is known as the Edge of Chaos, and empirical results show that it significantly improves performance. In \cite{yang2017meanfield, hayou21stable}, the authors derived similar results for the ResNet architecture, and showed that this architecture is \emph{placed} by-default on the Edge of Chaos for any choice of the variances of the initialization weights (Gaussian weights). In \cite{hayou2019impact}, the authors showed that an MLP that is initialized on the Edge of Chaos exhibits similar properties to ResNets, which might partially explain the benefits of the Edge of Chaos initialization.
    
    \item \emph{Gaussian process behaviour}: Multiple papers (e.g. \cite{neal, lee_gaussian_process, yang_tensor3_2020, matthews, hron20attention}) studied the weak limit of neural networks when the width goes to infinity. The results show that a randomly initialized neural network (with Gaussian weights) has a similar behaviour to that of a Gaussian process, for a wide range of neural architectures, and under mild conditions on the activation function. In \cite{lee_gaussian_process}, the authors leveraged this result and introduced the neural network Gaussian process (NNGP), which is a Gaussian process model with a neural kernel that depends on the architecture and the activation function. Bayesian regression with the NNGP showed that NNGP surprisingly achieves performance close to the one achieved by an SGD-trained finite-width neural network.
    
    The large depth limit of this Gaussian process was studied in \cite{hayou21stable}, where the authors showed that with proper scaling, the infinite-depth (weak) limit is a Gaussian process with a universal kernel\footnote{A kernel is called universal when any continuous function on some compact set can be approximated arbitrarily well with kernel features.}.
    \item \emph{Neural Tangent Kernel (NTK)}: the infinite-width limit of the NTK is the so-called NTK regime or Lazy-training regime. This topic has been extensively studied in the literature. The optimization and generalization properties (and some other aspects) of the NTK have been studied in \cite{Liu2022connecting, arora2019finegrained, seleznova2022ntk, hayou2019trainingdynamicsNTK}. The large depth asymptotics of the NTK have been studied in \citep{hayou_ntk, hayou2022curse, jacot2019freeze, xiao2020disentangling}. We refer the reader to \cite{jacot2022thesis} for a comprehensive discussion on the NTK.

    \item \emph{Tensor programs}: It is worth mentioning that a series of works called \emph{Tensor Programs} studied the dynamics of infinite-width limit of finite-depth general neural networks both at initialization and at finite training step $t$ with gradient descent \citep{yang2019tensor_i, yang_tensor3_2020, yang2019scaling, yang2021tensor_iv}.
    
    \item \emph{Others}: the theory of infinite-width neural networks have also been utilized for network pruning  \citep{hayou_pruning}, regularization  \citep{vladimirova19understanding, hayou2021stochasticdepth}, feature learning \citep{hayou_eh}, and ensembling methods \citep{he2020ntkensembles}.
\end{itemize}

\subsection{Infinite-depth limit}

 \paragraph{Infinite-width-then-infinite-depth limit.} In this case, the width of the neural network is taken to infinity first, followed by the depth. This is known as the infinite-depth limit of infinite-width neural networks. This limit has been widely used to study various aspects of neural networks, such as analyzing neural correlations and deriving the Edge of Chaos initialization scheme \citep{samuel2017, poole}, investigating the impact of the activation function \citep{hayou2019impact}, and analyzing the behavior of the Neural Tangent Kernel (NTK) \citep{hayou_ntk, xiao2020disentangling}. 
    
\paragraph{The joint infinite-width-and-depth limit.} In this case, the depth-to-width ratio is fixed\footnote{Other works consider the case when the depth-to-width ratio converge to a constant instead of being fixed.}, the width and depth are jointly taken to infinity. There are a limited number of studies that have examined the joint width-depth limit. For example, in \citep{li21loggaussian}, the authors demonstrated that for a specific form of residual neural networks (ResNets), the network output exhibits a (scaled) log-normal behavior in this joint limit, which is distinct from the sequential limit where the width is taken to infinity first followed by the depth, in which case the distribution of the network output is asymptotically normal (\citep{samuel2017, hayou2019impact}). Furthermore, in \citep{li2022sde}, the authors studied the covariance kernel of a multi-layer perceptron (MLP) in the joint limit and found that it weakly converges to the solution of a Stochastic Differential Equation (SDE). In \cite{Hanin2020Finite}, it was shown that in the joint limit case, the Neural Tangent Kernel (NTK) of an MLP remains random when the width and depth jointly go to infinity, which is different from the deterministic limit of the NTK when the width is taken to infinity before depth \citep{hayou_ntk}. In \citep{hanin2022correlation, hanin2019finitewidth}, the authors explored the impact of the depth-to-width ratio on the correlation kernel and the gradient norms in the case of an MLP architecture and found that this ratio can be interpreted as an effective network depth. Similar results have been discussed in \citep{zavatone2021exact, noci2021precise}.    
    
\paragraph{Infinite-depth limit of finite-width neural networks.}  In both previous limits, the width of the neural network is taken to infinity, either in isolation or jointly with the depth. However, it is natural to question the behavior of networks where the width is fixed and the depth is taken to infinity. For example, in \cite{hanin2019finitewidth}, it was shown that neural networks with bounded width are still universal approximators, motivating the examination of finite-width large depth neural networks. The limiting distribution of the network output at initialization in this scenario has been investigated in the literature. In \cite{peluchetti2020resnetdiffusion}, it was demonstrated that for a specific ResNet architecture, the pre-activations converge weakly to a diffusion process in the infinite-depth limit. This a simple corollary of existing results in stochastic calculus on the convergence of Euler-Maruyama disctretization schemes to continuous Stochastic Differential Equations. Other recent work by \cite{hayou2022on} examined the impact of the activation function on the distribution of the pre-activation, and characterized the distribution of the post-activation norms in this limit.

\section{Review of Stochastic Calculus}\label{appendix:stochastic_calculus}

In this section, we present the mathematical framework for the study of stochastic differential equations (SDEs). We consider a filtered probability space $(\Omega, \mathcal{F}, \mathbb{P}, (\mathcal{F}t){t \geq 0})$, where $\Omega$ is the sample space, $\mathcal{F}$ is the sigma-algebra of events, $\mathbb{P}$ is the probability measure, and $(\mathcal{F}t){t \geq 0}$ is the natural filtration of a standard $n$-dimensional Brownian motion $B$. This framework allows us to study the evolution of a stochastic process $X$ over time, by considering the events that are measurable up to a given time $t$. Specifically, we focus on the class of \ito processes, which are defined through a specific type of stochastic differential equation.

\subsection{Existence and uniqueness}
\begin{definition}[It$\hat{o}$ diffusion process]\label{def:ito_process}
A stochastic process $(X_t)_{t \in [0,T]}$ valued in $\reals^n$ is called an It$\hat{o}$ diffusion process if it can be expressed as 
$$
X_t = X_0 + \int_{0}^t \mu_s ds + \int_{0}^t \sigma_s dB_s,
$$
where $B$ is a $n$-dimensional Brownian motion and $\sigma_t \in \reals^{n \times n}, \mu \in \reals^n$ are predictable processes satisfying $\int_{0}^T (\|\mu_s\|_2 + \|\sigma_s \sigma_s^\top\|_2) ds < \infty$ almost surely.
\end{definition}

The following result gives conditions under which a strong solution of a given SDE exists, and is unique.

\begin{thm}[Thm 3.1 and Lemma 3.2 in \cite{xuerong2008}]\label{thm:existence_and_uniqueness}
Let $n \geq 1$, and consider the following SDE
$$
dX_t = \mu(t, X_t) dt + \sigma(t, X_t) dB_t, \quad X_0 \in L_2,
$$
where $B$ is a $m$-dimensional Brownian process for some $m\geq 1$, and $\mu : \reals^+ \times \reals^n \to \reals^n$ and $\sigma:\reals^+ \times \reals^n \to \reals^{n \times m}$ are measurable functions satisfying
\begin{enumerate}
    \item There exists a constant $K>0$ such that for all $t \geq 0$, $x,x' \in \reals^n$
    $$
\|\mu(t, x) - \mu(t,x')\| + \|\sigma(t, x) - \sigma(t,x')\| \leq K \|x - x'\|.
$$
\item There exists a constant $K'>0$ such that for all $t \geq 0$, $x \in \reals^n$
    $$
\|\mu(t, x)\| + \|\sigma(t, x) \| \leq K' (1 +\|x\|).
$$
\end{enumerate}
Then, for all $T \geq 0$, there exists a unique strong solution of the SDE above, and it satisfies the following 
$$
\E \sup_{0 \leq t \leq T} \|X_t\|^2 \leq C (1 + \E\|X_0\|^2),
$$
where $C$ is a constant that depends only on $K$, $K'$, and $T$.
\end{thm}

\subsection{\ito's lemma}
The following result, known as \ito's lemma, is a classic result in stochastic calculus. We state a version of this result from \cite{touzi2018}. Other versions and extensions exist in the literature (e.g. \citet{Ingersoll1987, Oksendal2003, kloeden}).
\begin{lemma}[\ito's lemma, Thm 6.7 in \cite{touzi2018}]\label{lemma:ito}
Let $X_t$ be an It$\hat{o}$ diffusion process (\cref{def:ito_process}) of the form
$$
dX_t = \mu_t dt + \sigma_t dB_t, t\in [0,T], X_0 \sim \nu
$$
where $\nu$ is some given distribution. Let $g: \reals^+ \times \reals^n \to \reals$ be $\mathcal{C}^{1,2}([0,T], \reals^n)$ (i.e. $\mathcal{C}^1$ in the first variable $t$ and $\mathcal{C}^2$ in the second variable $x$). Then, with probability $1$, we have that 
$$
f(t, X_t) = f(0,X_0) + \int_{0}^t \nabla_x f(s,X_s) \cdot dX_s + \int_{0}^t \left( \partial_t f(s,X_s) + \frac{1}{2} \textup{Tr}\left[ \sigma_s^\top \nabla_x^2 f(s,X_s) \sigma_s\right] \right) ds,
$$
where $\nabla_x f$ and $\nabla_x^2 f$ refer to the gradient and the Hessian, respectively. This can also be expressed as an SDE
$$
d f(t, X_t) =  \nabla_x f(t,X_t) \cdot dX_t + \left( \partial_t f(t,X_t) + \frac{1}{2} \textup{Tr}\left[ \sigma_t^\top \nabla_x^2 f(t,X_t) \sigma_t\right] \right) dt.
$$
\end{lemma}

\subsection{Convergence of Euler's scheme to the SDE solution}
The following result gives a convergence rate of the Euler discretization scheme to the solution of the SDE.
\begin{thm}[Corollary of Thm 7.3 in \cite{xuerong2008}]\label{thm:convergence_euler_xueorong}
Let $d \geq 1$ and consider the $\reals^d$-valued ito process $X$ (\cref{def:ito_process}) given by 
$$
X_t = X_0 + \int_{0}^t \mu(s, X_s) ds + \int_{0}^t \sigma(s, X_s) dB_s,
$$
where $B$ is a $m$-dimensional Brownian motion for some $m \geq 1$, $X_0$ satisfies $\E \|X_0\|^2 < \infty$, and $\mu: \reals^+ \times \reals^d \to \reals^d$ and $\sigma: \reals^+ \times \reals^d \to \reals^{d \times m}$ are measurable functions satisfying the following conditions:
\begin{enumerate}
    \item There exists a constant $K>0$ such that for all $t \in \reals, x,x' \in \reals^d$,
    $$
    \|\mu(t,x) - \mu(t,x')\|^2 + \|\sigma(t,x) - \sigma(t,x')\|^2 \leq \bar{K} \|x - x'\|^2.
    $$
    \item There exists a constant $K'>0$ such that for all $t \in \reals, x \in \reals^d$
    $$
    \|\mu(t,x)\|^2 + \|\sigma(t,x)\|^2 \leq K( 1 + \|x\|^2).
    $$
\end{enumerate}
Let $\delta \in (0,1)$ such that $\delta^{-1} \in \mathbb{N}$ (integer), and consider the times $t_k = k \delta$ for $k \in \{1, \dots, \delta^{-1}\}$. Consider the Euler discretization scheme given by 
$$
\bar{X}^i_{k+1} = \bar{X}^i_{k} + \mu^i(t_k, \bar{X}^k_{n}) \delta + \sum_{j=1}^m \sigma^{i,j}(t_k, \bar{X}^k_{n}) \Delta B^j_k, \quad \bar{X}^i_0 = X^i_0,
$$
where $\bar{X}^i, \mu^i, \sigma^{i,j}$ denote the coordinates of these vectors for $i \in [d], j\in[m]$, and $\Delta B^j_k = B^j_{k+1} - B^j_{k} \sim \normal(0, \delta)$. Then, we have that 
$$
\E \sup_{t \in [0,1]} \|X_{t} - \bar{X}_{\lfloor t \delta^{-1}\rfloor}\|^2 \leq C \, \delta,
$$
where $C = 80 K \bar{K} ( 1 + (1 + 3 \E \|X_0\|^2) \exp(6 K)) \exp(20 \bar{K})$.
\end{thm}

\begin{proof}
The proof is straightforward by taking $T = 1 $ and $t_0 = 0$ in Thm 7.3 in \cite{xuerong2008}.
\end{proof}

Using this result, we prove the following width-uniform convergence result for infinite-depth, which is crucial to our results.

\begin{thm}[Width-uniform convergence]\label{thm:width_uniform_euler_single}
Assume that the activation function $\phi$ is Lipschitz on $\reals$ with Lipschitz constant $\zeta > 0$ and that $\phi(0) = 0$, and let $a \in \reals^d$ be a non-zero vector. Consider the process $X_t$ the solution of the following SDE 
\begin{equation}
    dX_t = \frac{1}{\sqrt{n}}\|\phi(X_t)\| dB_t, \quad X_0 = W_{in} a,
\end{equation}
where $(B_t)_{t\geq 0}$ is a Brownian motion (Wiener process), and let $\bar{X}$ be its Euler scheme as in \cref{thm:convergence_euler_xueorong}. Then, we have the following width-uniform bound on the discretization error:
$$
\sup_{n \geq 1} n^{-1} \, \E \sup_{t \in [0,1]} \|X_{t} - \bar{X}_{\lfloor t \delta^{-1}\rfloor}\|^2 \leq C' \, \delta,
$$
where $C' = 80 \zeta^4 (1 + (1 + 3 d^{-1} \|a\|^2) \exp(6 \zeta^2)) \exp(20 \zeta^2)$.
\end{thm}

\begin{proof}
The key observation in this proof is that the constant $C$ in \cref{thm:convergence_euler_xueorong} scales linearly with width. Indeed, in this case, the volatility term is given by $\sigma(x) = \frac{1}{\sqrt{n}}\| \phi(x)\| I_n$, which satisfies the linear growth condition
$$
\|\sigma(x)\| =  \frac{1}{\sqrt{n}} \|\phi(x)\| \|I_n\| = \|\phi(x)\| \leq \zeta \|x\|,
$$
where we have used the fact that $\|I_n\| = \sqrt{\textup{Tr}(I_n I_n^\top)} = \sqrt{n}$\footnote{In (almost) all the results on the existence, uniqueness, and Euler schemes in stochastic calculus, the default matrix norm is the Frobenius norm.}.
Moreover, for any $x,x' \in \reals^n$, we have that
$$
\|\sigma(x) - \sigma(x')\| \leq \left| \frac{1}{\sqrt{n}} \|\phi(x)\| - \frac{1}{\sqrt{n}} \|\phi(x)\|\right| \|I_n\| \leq \zeta \|x - x'\|,
$$
Hence, in this case we can set $\bar{K} = K = \zeta^2$. We conclude by observing that $\E\|X_0\|^2 = n d^{-1} \|a\|^2$ and using \cref{thm:convergence_euler_xueorong}.
\end{proof}

The result of \cref{thm:width_uniform_euler_single} can be generalized to the case of multiple inputs as we show in the next result. We omit the proof here as this result is not necessary for the proofs of the main results.

\begin{thm}
Let $a_1, a_2, \dots, a_k \in \reals^d$ be non-zero inputs, and assume that the activation function $\phi$ is Lipschitz on $\reals$ and that $\phi(0) = 0$. Consider the process $X_t^k$, the solution of the following SDE 
\begin{equation}
    d\bm{X}^k_t = \frac{1}{\sqrt{n}}\Sigma(\bm{X}^k_t)^{1/2} d\bm{B}_t, \quad \bm{X}^k_0 = ((W_{in} a_1)^\top, \dots, (W_{in} a_k)^\top)^\top,
\end{equation}
where $(\bm{B}_t)_{t\geq 0}$ is an $kn$-dimensional Brownian motion (Wiener process), independent from $W_{in}$, and $\Sigma(\bm{X}^k_t) $ is the covariance matrix given by 
\[
  \Sigma(\bm{X}^k_t) = \left[\begin{array}{ c | c | c | c}
    \alpha_{1,1} I_n & \alpha_{1,2} I_n & \dots & \alpha_{1,k} I_n\\
    \hline
    \alpha_{2,1} I_n & \alpha_{2,2} I_n & \dots & \alpha_{2,k} I_n\\
    \hline
    \vdots & \vdots & \vdots & \vdots\\
    \alpha_{k,1} I_n & \dots & \dots & \alpha_{k,k} I_n\\
  \end{array}\right],
\]
where $\alpha_{i,j} = \langle \phi(\bm{X}_{t}^{k, i}), \phi(\bm{X}_{t}^{k, j}) \rangle$, with $(X_{t}^{k, 1}{}^\top, \dots, X_{t}^{k, k}{}^\top )^\top \overset{def}{=} \bm{X}_t^k$.

Let $\bar{\bm{X}}{}^{k}$ be its Euler scheme as in \cref{thm:convergence_euler_xueorong}. Then, we have the following width-uniform bound on the discretization error:
$$
\sup_{n \geq 1} (k n)^{-1} \, \E \sup_{t \in [0,1]} \|\bm{X}^k_{t} - \bar{\bm{X}}^k_{\lfloor t \delta^{-1}\rfloor}\|^2 \leq C' \, \delta,
$$
where $C' = 80 \zeta^4 (1 + (1 + 3 d^{-1} \|a\|^2) \exp(6 \zeta^2)) \exp(20 \zeta^2)$.
\end{thm}




\subsection{Convergence of Particles to the solution of Mckean-Vlasov process}
The next result gives sufficient conditions for the system of particles to converge to its mean-field limit, known as the Mckean-Vlasov process.

\begin{thm}[Uniform Mckean-Vlasov process]\label{thm:convergence_mckean_new}
Let $d \geq 1$ and consider the $\reals^d$-valued ito process $X$ (\cref{def:ito_process}) given by 
$$
d X_t = \sigma(\nu^n_t) dB_t, \quad X_0=W_{in} a,
$$
where $B$ is a $d$-dimensional Brownian motion, $W_{in}^{ij} \sim \normal(0,1/d)$, $a \in \reals^d$ and $a\neq 0$, $\nu^n_t \overset{def}{=} \frac{1}{d} \sum_{i=1}^d \delta_{\{X^i_t\}}$ is the empirical distribution of the coordinates of $X_t$, and $\sigma$ is real-valued given by $\sigma(\nu) = \left(\int \phi(y)^2 d\nu(y) \right)^{1/2}$ for any distribution $\nu$, where $\phi$ is the ReLU activation function. Then, for all $T \in \reals^+$, we have that
$$
\sup_{i \in [n]} \E \left( \sup_{t \leq T} |X^i_t - \tilde{X}^i_t|^2 \right) = \bigO(n^{-1}),
$$
where $\tilde{X}^i$ is the solution of the following Mckean-Vlasov equation
$$
d\tilde{X}^i_t = \sigma( \nu^i_t) dB^i_t = \frac{\|a\|}{\sqrt{2d}} \exp(t/4) dB^i_t, \quad \tilde{X}^i_0 = X^i_0,
$$
where $\nu^i_t$ is the distribution of $\tilde{X}^i$. The constant in the $\bigO$ depends only on $T$ and the norm of $a$.
\end{thm}

\begin{proof}
The first part of the proof is similar to that of Theorem 3 in \citep{jourdain2007}. In the second part, we use a concentration argument to control the deviations of the volatility term which allow us to conclude. \\

Let $\tilde{v}^n_t$ denote the empirical distribution of the independent processes $\tilde{X}^i_t, i\in [n]$ defined in the statement of the theorem. Let $t \in [0,1]$.
Following \citep{jourdain2007}, for some $i \in [n]$, using Doob's inequality, there exists a universal constant $C>0$ such that
\begin{align*}
\E \left( \sup_{s \leq t} |X^i_s - \tilde{X}^i_s|^2 \right) &\leq C \int_{0}^t \E |\sigma(\nu^n_s) - \sigma(\nu_s)|^2 ds\\
&\leq C \int_{0}^t \E |\sigma(\nu^n_s) - \sigma(\tilde{\nu}^n_s)|^2 ds + C \int_{0}^t \E |\sigma(\tilde{\nu}^n_s) - \sigma(\nu_s)|^2 ds.\\
\end{align*}

For the first term, we have that 
\begin{align*}
    \int_{0}^t \E |\sigma(\nu^n_s) - \sigma(\tilde{\nu}^n_s)|^2 ds &= \int_{0}^t \E \left| \frac{1}{\sqrt{n}} \|\phi(X_s)\| - \frac{1}{\sqrt{n}} \|\phi(\tilde{X}_s)\|\right|^2 ds\\
    &\leq  \frac{1}{n} \int_{0}^t \E\|\phi(X_s)- \phi(\tilde{X}_s)\|^2 ds\\
    &\leq  \int_{0}^t  \E \left( \sup_{r \leq s}|X^i_r- \tilde{X}^i_r|^2 \right) ds,
\end{align*}
where we have used the exchangeability of the couples $(X^i_t, \tilde{X}^i_t)$ (across $i$) and the Lipschitz property of $\zeta$. Therefore, using Gronwall's lemma, there exists a constant $C'>0$ (independent of $i$) such that 
$$
\E \left( \sup_{s \leq t} |X^i_s - \tilde{X}^i_s|^2 \right) \leq C' \int_{0}^t \E |\sigma(\tilde{\nu}^n_s) - \sigma(\nu_s)|^2 ds.
$$
Since the bound is uniform in $i$, we then have 
$$
\sup_{i \in [n]} \E \left( \sup_{s \leq t} |X^i_s - \tilde{X}^i_s|^2 \right) \leq C' \int_{0}^t \E |\sigma(\tilde{\nu}^n_s) - \sigma(\nu_s)|^2 ds.
$$

Thus, it suffices to show that the right hand side is of order $n^{-1}$ to conclude. Let us first show that the volatility of the process $\tilde{X}^i_t$ is given by $
\sigma(\nu^i_t) = \frac{\|a\|}{\sqrt{2d}} \exp(t/4).
$
We have that $d\tilde{X}^i_t = \sigma(\nu^i_t) dB^i_t$. A simple application of \ito's lemma (\cref{lemma:ito}) yields 
$$
d \E(\tilde{X}^i_t)^2 = \frac{1}{2} \E(\tilde{X}^i_t)^2 dt,
$$
where we have used the fact that with ReLU $\E(\phi(\tilde{X}^i_t)^2) = \frac{1}{2} \E(\tilde{X}^i_t)^2.$ Therefore, we obtain $\E(\tilde{X}^i_t)^2 = \E(\tilde{X}^i_0)^2 \exp(t/2) = \frac{\|a\|^2}{d} \exp(t/2)$. Thus, the volatility term is given by stated formula. Notice that $\hat{X}^i_t$ has a normal distribution in this case.

We now use Hoeffding's inequality for random variables with sub-exponential growth to control the deviations of $\sigma(\tilde{\nu}^n_s)^2$. We have

\begin{align*}
    \mathbb{P}\left(\sigma(\tilde{\nu}^n_s)^2 \leq \frac{\|a\|^2}{4 d} \right) &\leq \mathbb{P}\left(\sigma(\tilde{\nu}^n_s)^2 \leq \sigma(\nu_s)^2/2 \right) \\
&= \mathbb{P}\left(\sigma(\tilde{\nu}^n_s)^2 - \sigma(\nu_s)^2 \leq - \sigma(\nu_s)^2/2 \right)\\
&\leq 2\exp(- n c),
\end{align*}
where $c>0$ is a constant that depends only on the moments of $\phi(\tilde{X}^i_t)$ which can be upper-bounded uniformly for $t \in [0,T]$. Define the event $\mathcal{H}_n = \{\sigma(\tilde{\nu}^n_s)^2 \leq \frac{\|a\|^2}{4d}\}$ and let $\bar{\mathcal{H}}_n$ denote its complementary event. This yields for all $s \in [0,T]$
\begin{align*}
\E |\sigma(\tilde{\nu}^n_s) - \sigma(\nu_s)|^2  &= \E\, \ind_{\mathcal{H}_n} |\sigma(\tilde{\nu}^n_s) - \sigma(\nu_s)|^2 +  \E \ind_{\bar{\mathcal{H}}_n} |\sigma(\tilde{\nu}^n_s) - \sigma(\nu_s)|^2\\
&\leq \frac{2 \|a\|^2}{d} \exp\left(-n \, c + \frac{s}{2}\right)  + \left(\frac{d}{4}\right)^{1/4} \E|\sigma(\tilde{\nu}^n_s)^2 - \sigma(\nu_s)^2| ^2\\
&\leq \frac{2 \|a\|^2}{d} \exp\left(-n \, c + \frac{s}{2}\right)  + \left(\frac{\sqrt{d}}{2 \|a\|}\right) \frac{\E \phi(\tilde{X}^1_s)^4}{n},
\end{align*}
where we have use the fact that $|\sqrt{z} - \sqrt{z'}| \leq \frac{1}{2\sqrt{z_0}} |z- z'| $ for $z,z' \geq z_0 >0$. Since $\tilde{X}^1_s$ is a zero-mean Gaussian with a variance that depends only on $s$, we can therefore conclude that there exists $C''$ independent of $n$ and $i \in [n]$ such that 
$$\
\sup_{i \in [n]} \E \left( \sup_{s \leq t} |X^i_s - \tilde{X}^i_s|^2 \right) \leq C'' \, n^{-1},
$$
which concludes the proof.
\end{proof}

\subsection{Other results from probability and stochastic calculus}
The next trivial lemma has been opportunely used in \cite{li21loggaussian} to derive the limiting distribution of the network output (multi-layer perceptron) in the joint infinite width-depth limit. This simple result will also prove useful in our case of the finite-width-infinite-depth limit.
\begin{lemma}\label{lemma:gaussian_vec}
Let $W \in \reals^{n\times n}$ be a matrix of standard Gaussian random variables $W_{ij} \sim \normal(0,1)$. Let $v \in \reals^n$ be a random vector independent from $W$ and satisfies $\|v\|_2 = 1$ . Then, $W v \sim \normal(0, I)$.
\end{lemma}
\begin{proof}
The proof follows a simple characteristic function argument. Indeed, by conditioning on $v$, we observe that $Wv \sim \normal(0, I)$. Let $u \in \reals^n$, we have that \begin{align*}
    \E_{W,v}[e^{i \langle u, Wv\rangle}]  &=  \E_v[ \E_W[e^{i \langle u, Wv\rangle}| v]] \\
    &= \E_v[ e^{-\frac{\|u\|^2}{2}}] \\
    &= e^{-\frac{\|u\|^2}{2}}.\\
\end{align*}
This concludes the proof as the latter is the characteristic function of a random Gaussian vector with Identity covariance matrix.
\end{proof}

\section{Some technical results for the proofs}

\begin{prop}\label{prop:convergence_law_width_uniform}
Assume that the activation function $\phi$ is Lipschitz on $\reals$ and let $a \in \reals^d$ with $a \neq 0$. Then, in the limit $L \to \infty$, the process $X^L_t(a) = Y_{\lfloor t L\rfloor}(a)$, $t\in [0,1]$, converges in distribution to the solution of the following SDE 
\begin{equation}\label{eq:main_sde}
    dX_t(a) = \frac{1}{\sqrt{n}}\|\phi(X_t(a))\| dB_t, \quad X_0(a) = W_{in} a,
\end{equation}
where $(B_t)_{t\geq 0}$ is a Brownian motion (Wiener process). Moreover, we have that
$$
\sup_{n \geq 1} \sup_{ 1 \leq t \leq 1} \mathcal{W}_1(\mu^t_{n,L}, \mu^t_{n,\infty})  \leq C L^{-1/2},
$$
where $\mu^t_{n,L}(a)$ is the distribution of $Y^i_{\lfloor t L\rfloor}(a)$, $\mu^t_{n, \infty}(a)$ is the distribution $X^i_t(a)$ (for any $i$ since the coordinates are identically distributed), and $C$ is a constant that depends only on $d$ and $\|a\|$.\\
\end{prop}

\begin{proof}
The proof is based on \cref{thm:width_uniform_euler_single} in the appendix. It remains to express \cref{eq:resnet}  in the required form and make sure all the conditions are satisfied for the result to hold. To alleviate the notation, we denote $Y_l:= Y_l(a)$. Using \cref{lemma:gaussian_vec}, we can write \cref{eq:resnet} as
$$
Y_l = Y_{l-1} + \frac{1}{\sqrt{L}} \sigma(Y_{l-1}) \zeta^L_{l-1},
$$
where $\sigma(y) \overset{def}{=} \frac{1}{\sqrt{n}} \|\phi(y)\|$ for all $y \in \reals^n$ and $\zeta^L_l$ are \iid random Gaussian vectors with distribution $\normal(0, I)$. This is equal in distribution to the Euler scheme of SDE \cref{eq:main_sde}. Since $\sigma$ trivially inherits the Lipschitz or local Lipschitz properties of $\phi$, we conclude for the convergence using \cref{thm:width_uniform_euler_single}.\\

Now let $\Psi$ be $1$-Lipschitz. We have that 
$$
|\E \Psi(Y_{\lfloor t L\rfloor}) - \E \Psi(X_t)| \leq  \E \|\bar{X}_{\lfloor t L\rfloor} - X_t\| \leq C L^{-1/2}.
$$
where $\bar{X}$ is the Euler scheme as in \cref{thm:width_uniform_euler_single}, and where we have used the fact that $Y_{\lfloor t L\rfloor}$ and $\bar{X}_{\lfloor t L\rfloor}$ have the same distribution, coupled with the Cauchy-Schwartz inequality. Since $C$ depends only on $d$ and $\|a\|$, the conclusion is straightforward.
\end{proof}

\begin{prop}\label{prop:convergence_law_width_uniform_multiple}
Assume that the activation function $\phi$ is Lipschitz on $\reals$ and let $a, b \in \reals^d$ with $a, b \neq 0$ and $a \neq b$. Then, there exists two $n$-dimensional Brownian motions $B_t(a)$ and $B_t(b)$ and a discretized Euler scheme $(\bar{X}(a))$ and $(\bar{X}(b))$ such that  for any $t\in [0,1]$, the processes $(Y_{\lfloor t L\rfloor}(a), Y_{\lfloor t L\rfloor}(b))$ have the same distribution as $(\bar{X}_{\lfloor t L\rfloor}(a), \bar{X}_{\lfloor t L\rfloor}(b))$ and $\bar{X}_{\lfloor t L\rfloor}(a)$ and $\bar{X}_{\lfloor t L\rfloor}(b)$ converge (in $L_2$) to the solutions of the following SDEs 
\begin{equation}
\begin{aligned}
    dX_t(a) &= \frac{1}{\sqrt{n}}\|\phi(X_t(a))\| dB_t(a), \quad X_0(a) = W_{in} a,\\
    dX_t(b) &= \frac{1}{\sqrt{n}}\|\phi(X_t(b))\| dB_t(b), \quad X_0(b) = W_{in} b,
\end{aligned}
\end{equation}
Moreover, we have that
$$
\lim_{n \to \infty} \E \left[\frac{\langle X_{t}(a), X_{t}(b) \rangle}{n}\right] = q_t(a,b),
$$
where $q_t(a,b)$ is the solution of the following Ordinary Differential Equation
\begin{equation}
\begin{aligned}
\frac{d q_t(a,b)}{dt} &= \frac{1}{2} \frac{f(c_t(a,b))}{c_t(a,b)} q_t(a,b),\\
c_t(a,b) &= \frac{q_t(a,b)}{ \sqrt{q_t(a,a)} \sqrt{q_t(b,b)}},\\
q_0(a,b) &= \frac{\langle a, b \rangle}{d},
\end{aligned}
\end{equation}
where the function $f: [-1,1] \to [-1,1]$ is given by 
$$
f(z) = \frac{1}{\pi} ( z \arcsin(z) + \sqrt{1 - z^2}) + \frac{1}{2}z.
$$
\end{prop}

\begin{proof}
The proof is similar to that of \cref{prop:convergence_law_width_uniform}. The only difference lies the definition of the Gaussian vector $\zeta^L_l$. In this case, for $x \in \{a,b\}$, we have 
$$
Y_l(x) = Y_{l-1}(x) + \frac{1}{\sqrt{L}} \frac{1}{\sqrt{n}} \zeta^L_{l-1}(Y_{l-1}(x)),
$$
where $\zeta^L_{l-1}(Y_{l-1}(x)) \overset{def}{=} \sqrt{n} W_l \phi(Y_{l-1}(x))$. It is straightforward that we can write $\frac{1}{\sqrt{L}}\zeta^L_{l-1}(Y_{l-1}(x))$ as a Brownian increment $\Delta B_l(x) = L^{-1/2} \, \zeta^L_{l-1}(Y_{l-1}(x))$. Defining the Euler schemes $\bar{X}(a), \bar{X}(b)$ with the Brownian motions $(B_t(x))_{x \in \{a,b\}}$ yields that the concatenated vector $(Y_{\lfloor t L\rfloor}(a), Y_{\lfloor t L\rfloor}(b))$ has the same distribution as $(\bar{X}_{\lfloor t L\rfloor}(a), \bar{X}_{\lfloor t L\rfloor}(b))$. In particular, this implies that 
$$
\E \left[\frac{\langle \bar{X}_{\lfloor t L\rfloor}(a), \bar{X}_{\lfloor t L\rfloor}(b) \rangle}{n}\right] = \E \left[\frac{\langle Y_{\lfloor t L\rfloor}(a), Y_{\lfloor t L\rfloor}(b) \rangle}{n}\right].
$$

Now using \cref{thm:width_uniform_euler_single}, we know that for $x \in \{a,b\}$, 
$$
\sup_{n \geq 1} n^{-1} \, \E \sup_{t \in [0,1]} \|X_{t}(x) - \bar{X}_{\lfloor t L\rfloor}(x)\|^2 \leq C' \, \delta,
$$
where $C'$ depends only on the $\|x\|$ and $d^{-1}$. From this, and by observing that the $L_2$ norm of $X_{t}(x)$ and $\bar{X}_{\lfloor t L\rfloor}(x)$ are upperbounded (see \cref{thm:existence_and_uniqueness}), it is straightforward that 
\begin{align*}
\left| \E \left[\frac{\langle X_{t}(a), X_{t}(b) \rangle}{n}\right] -  \E \left[\frac{\langle Y_{\lfloor t L\rfloor}(a), Y_{\lfloor t L\rfloor}(b) \rangle}{n}\right] \right| \leq C L^{-1/2},
\end{align*}
where $C$ is a constant that depends only on $\|a\|, \|b\|$, and $d$. To conclude, we will take the width to infinity first then take the depth to infinity. Taking $n \to \infty$, then depth to $\infty$ (standard result, see Lemma 5 in \cite{hayou21stable}) yields 
$$
\lim_{L \to \infty} \lim_{n \to \infty} \E \left[\frac{\langle Y_{\lfloor t L\rfloor}(a), Y_{\lfloor t L\rfloor}(b) \rangle}{n}\right] = q_t(a,b),
$$
which concludes the proof.

\end{proof}

\section{Proof of \cref{thm:gaussian_limit}}\label{sec:proof_gaussian_limit}

\textbf{Theorem \ref{thm:gaussian_limit} }[Width/Depth uniform convergence of the pre-activations]\\
\emph{Let $a \in \reals^d$ such that $a \neq 0$. For $t \in [0,1]$ and $i \in [n]$ fixed, the random variable $(Y_{\lfloor t L\rfloor}(a))_{L \geq 1}$ converges weakly to a Gaussian random variable with law $\normal(0, v(t,a))$ in the limit of $\min(n,L) \to \infty$, where $v(t,a) = d^{-1} \|a\|^2 \exp(t)$.
Moreover, we have the following convergence rate
$$
\sup_{t \in [0,1]} \mathcal{W}_1(\mu^t_{n,L}(a), \mu^t_{\infty, \infty}(a)) \leq C \left(\frac{1}{\sqrt{n}} + \frac{1}{\sqrt{L}} \right)
$$
where $\mu^t_{n,L}(a)$ is the distribution of $Y^1_{\lfloor t L\rfloor}(a)$, $\mu^t_{\infty, \infty}(a)$ is the distribution $\normal(0, v(t,a))$, and $C$ is constant that depends only on $\|a\|$ and $d$.
}
\begin{proof}
The proof relies on a careful manipulation of the order of the depth and width limits. Unlike existing literature on the infinite-width-then-depth networks, we found that is much easier to control the convergence rate by looking at what happens when the $L$ diverges first, then control over $n$. This uses two main ingredients:
\begin{itemize}
    \item A new width-uniform convergence rate of the Euler discretization scheme of the infinite-depth SDE. We prove this in \cref{thm:width_uniform_euler_single}.
    \item A new particle convergence result to a  McKean-Vlasov process (Mean-Field limit). We prove this result in \cref{thm:convergence_mckean_new}.
\end{itemize}

Let $a \in\reals^d$ with $a \neq 0$. \\

\textbf{Part 1: Width-uniform infinite-depth limit.} Let $n \geq 1$ be fixed for now, and let us look at what happens in the infinite depth limit. Using \cref{prop:convergence_law_width_uniform}, we know that $Y^1_{\lfloor t L\rfloor}(a)$ converges in distribution to $X^1_t(a)$ with a width-uniform rate in terms of the Wasserstein distance
$$
\sup_{ 1 \leq t \leq 1} \mathcal{W}_1(\mu^t_{n,L}, \mu^t_{n,\infty})  \leq C L^{-1/2},
$$
where $C$ depends only on $d$ and $\|a\|$.\\

\textbf{Part 2: Taking the width to infinity.} The rest of the proof rely on a new technical result that we prove in \cref{thm:convergence_mckean_new}. The intuition is that the coordinates $(X^i(a)_t)_{1 \leq i \leq n}$ can be seen as interacting particles of some underlying mean-field process. This is known as Mckean-Vlasov process. Using \cref{thm:convergence_mckean_new} with $T=1$, we obtain that 
$$
\sup_{i \in [n]} \E \left( \sup_{0 \leq t \leq 1} |X^i_t(a) - \tilde{X}^i_t(a)|^2 \right) \leq C' n^{-1},
$$
where $\tilde{X}^i_t(a)$ is the solution of the SDE
$$
d \tilde{X}^i_t(a) = \frac{\|a\|}{\sqrt{2d}} \exp(t/4) dB_t^i, \tilde{X}^i_0(a) = X^i_0(a).
$$

This is a special SDE since all the marginal distributions are zero-mean Gaussians (sum of Brownian increments) with variance $\E \tilde{X}^i_t(a)^2 = \frac{\|a\|^2}{d} \exp(t/2)$.

In particular, $X^i(a)_t$ converges weakly to $\tilde{X}^i_t(a)$ in the limit of infinite width $n$. Combining the bound in Part 1 with the Mckean-Vlasov bound above, we obtain 
$$
\sup_{t \in [0,1]} \mathcal{W}_1(\mu^t_{n,L}(a), \mu^t_{\infty, \infty}(a)) \leq C \left(\frac{1}{\sqrt{n}} + \frac{1}{\sqrt{L}} \right)
$$
for some constant that depends only on $d$ and $\|a\|$, and where $\mu^t_{\infty, \infty}(a)$ is the distribution of $\tilde{X}^i_t(a) \sim \normal(0, d^{-1} \|a\|^2 \exp(t/2))$.

\end{proof}

\section{Proof of \cref{thm:covariance}}\label{sec:proof_covariance}
\begin{thm}[Neural Covariance]
Let $a, b \in \reals^d$ such that $a, b \neq 0$ and $a \neq b$. Then, we have the following 
$$
\sup_{t \in [0,1]} \left\| \frac{\langle Y_{\lfloor t L\rfloor}(a), Y_{\lfloor t L\rfloor}(b) \rangle}{n} - q_t(a,b)\right\|_{L_2} \leq C \left(\frac{1}{\sqrt{n}} + \frac{1}{\sqrt{L}} \right)
$$
where $C$ is a constant that depends only on $\|a\|$, $\|b\|$, and $d$, and $q_t(a,b)$ is the solution of the following Ordinary Differential Equation
\begin{equation}
\begin{aligned}
\frac{d q_t(a,b)}{dt} &= \frac{1}{2} \frac{f(c_t(a,b))}{c_t(a,b)} q_t(a,b),\\
c_t(a,b) &= \frac{q_t(a,b)}{ \sqrt{q_t(a,a)} \sqrt{q_t(b,b)}},\\
q_0(a,b) &= \frac{\langle a, b \rangle}{d},
\end{aligned}
\end{equation}
where the function $f: [-1,1] \to [-1,1]$ is given by 
$$
f(z) = \frac{1}{\pi} ( z \arcsin(z) + \sqrt{1 - z^2}) + \frac{1}{2}z.
$$
\end{thm}

\begin{proof}
Let $a, b \in \reals^d$ and $q_t$ be as in the statement of the theorem. Let $X_t(a)$ and $X_t(b)$ be the infinite-depth limits as in \cref{prop:convergence_law_width_uniform_multiple}, and let $\bar{X}(a), \bar{X}(b)$ be the corresponding Euler schemes. Using the fact that $(Y_{\lfloor t L\rfloor}(a), Y_{\lfloor t L\rfloor}(b))$ has the same law as $(\bar{X}_{\lfloor t L\rfloor}(a), \bar{X}_{\lfloor t L\rfloor}(b))$ , we trivially have 
$$
\E \left| \frac{\langle Y_{\lfloor t L\rfloor}(a), Y_{\lfloor t L\rfloor}(b) \rangle}{n} - q_t(a,b)\right|^2 = \E \left| \frac{\langle \bar{X}_{\lfloor t L\rfloor}(a), \bar{X}_{\lfloor t L\rfloor}(b) \rangle}{n} - q_t(a,b)\right|^2.
$$
We have the following upperbound
\begin{equation}\label{eq:upperbound_proof}
\begin{aligned}
\left\| \frac{\langle \bar{X}_{\lfloor t L\rfloor}(a), \bar{X}_{\lfloor t L\rfloor}(b) \rangle}{n} - q_t(a,b)\right\|_{L_2} &\leq  \left\| \frac{\langle \bar{X}_{\lfloor t L\rfloor}(a), \bar{X}_{\lfloor t L\rfloor}(b) \rangle}{n} - \frac{\langle X_{t}(a), X_{t}(b) \rangle}{n} \right\|_{L_2}\\
&+ \left\| \frac{\langle X_{t}(a), X_{t}(b) \rangle}{n} - \frac{\langle \tilde{X}_{ t}(a), \tilde{X}_{ t}(b) \rangle}{n} \right\|_{L_2}\\
&+ \left\| \frac{\langle \tilde{X}_{t}(a), \tilde{X}_{t}(b) \rangle}{n} - q_t(a,b) \right\|_{L_2},
\end{aligned}
\end{equation}
where $\tilde{X}(a), \tilde{X}(b)$ are the infinite-width limits of the processes $X(a), X(b)$ as in \cref{thm:convergence_mckean_new}. Let us deal with each term in this bound. 

\begin{itemize}
    \item First term: from \cref{thm:width_uniform_euler_single} and standard upperbounds on the second moments (\cref{thm:existence_and_uniqueness}), we have that 
    $$
     \left\| \frac{\langle \bar{X}_{\lfloor t L\rfloor}(a), \bar{X}_{\lfloor t L\rfloor}(b) \rangle}{n} - \frac{\langle X_{t}(a), X_{t}(b) \rangle}{n} \right\|_{L_2} \leq C_1 \,L^{-1/2},
    $$
    where $C_1$ is a constant that depends only on $\|a\|, \|b\|$, and $d$.
    
    \item Second term: from \cref{thm:convergence_mckean_new}, there exists a constant $C_2$ such that 
    $$
     \left\| \frac{\langle X_{t}(a), X_{t}(b) \rangle}{n} - \frac{\langle \tilde{X}_{ t}(a), \tilde{X}_{ t}(b) \rangle}{n} \right\|_{L_2} \leq C_2 n^{-1/2},
    $$
    where $C_2$ depends only on $\|a\|, \|b\|$, and $d$.
    
    \item Third term: from \cref{prop:convergence_law_width_uniform_multiple}, we know that $\lim_{n \to \infty} \E \left[\frac{\langle X_{t}(a), X_{t}(b) \rangle}{n}\right] = q_{t}(a,b)$. Using the bound above on the second term, we obtain that $\lim_{n \to \infty} \E \left[\frac{\langle \tilde{X}_{t}(a), \tilde{X}_{t}(b) \rangle}{n}\right] = q_{t}(a,b)$. Now the key observation is that 
    $$\frac{\langle \tilde{X}_{t}(a), \tilde{X}_{t}(b) \rangle}{n} = \frac{1}{n} \sum_{i=1}^n \tilde{X}^i_{t}(a) \tilde{X}^i_{t}(b),$$
    and the terms in the sum above are iid with mean $q_t(a,b)$. Therefore, 
    $$
    \left\| \frac{\langle \tilde{X}_{t}(a), \tilde{X}_{t}(b) \rangle}{n} - q_t(a,b) \right\|_{L_2} = \left(\E \left| \frac{\langle \tilde{X}_{t}(a), \tilde{X}_{t}(b) \rangle}{n} - q_t(a,b) \right|^2\right)^{1/2} \leq (\E(\tilde{X}^1_t(a)\tilde{X}^1_t(b))^2)^{1/2} \, n^{-1/2}.
    $$
    Observe that $\E(\tilde{X}^1_t(a)\tilde{X}^1_t(b))^2$ can be bounded with a constant $C_3$ depends only on $\|a\|, \|b\|$, and $d$.
\end{itemize}
We conclude by combining the three bounds above.
\end{proof}

\section{ Proof of \cref{thm:correlation}}

\textbf{Theorem \ref{thm:correlation}. }[Neural correlation]\\
\emph{Under the same conditions of \cref{thm:covariance}, we have the following 
$$
\sup_{t \in [0,1]} \left\| \hat{c}_t(a,b) - c_t(a,b)\right\|_{L_2} \leq C' \left(\frac{1}{\sqrt{n}} + \frac{1}{\sqrt{L}} \right)
$$
where $C'$ is a constant that depends only on $\|a\|$, $\|b\|$, and $d$, and $\hat{c}_t(a,b) = \frac{\langle Y_{\lfloor t L\rfloor}(a), Y_{\lfloor t L\rfloor}(b) \rangle}{\|Y_{\lfloor t L\rfloor}(a)\| \|Y_{\lfloor t L\rfloor}(b)\|}$ is the neural correlation kernel, and $c_t(a,b)$ is defined in \cref{thm:covariance}.}

\begin{proof}
 Let $a$ and $b$ be as stated in the theorem. We have the following
$$
\left\| \hat{c}_t(a,b) - c_t(a,b)\right\|_{L_2} \leq \left\| \frac{\hat{q}_t(a,b)}{\sqrt{\hat{q}_t(a,a) \hat{q}_t(b,b)}} - \frac{q_t(a,b)}{\sqrt{\hat{q}_t(a,a) \hat{q}_t(b,b)}}\right\|_{L_2} + \left\| \frac{q_t(a,b)}{\sqrt{\hat{q}_t(a,a) \hat{q}_t(b,b)}} - \frac{q_t(a,b)}{\sqrt{q_t(a,a) q_t(b,b)}}\right\|_{L_2}.
$$
Using Markov's inequality along with \cref{thm:covariance}, it is straightforward that there exists a constant $C_1$ that depends only on $\|a\|$, $\|b\|$, and $d$ such that  
$$
\mathbb{P}\left(\hat{q}_t(a,a) < \frac{q_t(a,a)}{2}\right) \leq C_1 \min(n,L)^{-1},
$$
and
$$
\mathbb{P}\left(\hat{q}_t(b,b) < \frac{q_t(b,b)}{2}\right) \leq C_1 \min(n,L)^{-1}.
$$
Let $A$ denote the event $\{\hat{q}_t(a,a) \geq  \frac{q_t(a,a)}{2}\} \cup \{\hat{q}_t(b,b) \geq  \frac{q_t(b,b)}{2}\} $. With this, we obtain the following upperbound
\begin{align*}
\left\| \frac{\hat{q}_t(a,b)}{\sqrt{\hat{q}_t(a,a) \hat{q}_t(b,b)}} - \frac{q_t(a,b)}{\sqrt{\hat{q}_t(a,a) \hat{q}_t(b,b)}}\right\|_{L_2} &\leq \left\| \ind_A \, \left(\frac{\hat{q}_t(a,b)}{\sqrt{\hat{q}_t(a,a) \hat{q}_t(b,b)}} - \frac{q_t(a,b)}{\sqrt{\hat{q}_t(a,a) \hat{q}_t(b,b)}} \right)\right\|_{L_2}\\
&+\left\| \ind_{A^c} \, \left(\frac{\hat{q}_t(a,b)}{\sqrt{\hat{q}_t(a,a) \hat{q}_t(b,b)}} - \frac{q_t(a,b)}{\sqrt{\hat{q}_t(a,a) \hat{q}_t(b,b)}} \right)\right\|_{L_2}\\
&\leq \frac{2}{\sqrt{q_t(a,a) q_t(b,b)}} \left\| \hat{q}_t(a,b) - q_t(a,b)\right\|_{L_2} + C_2 \mathbb{P}(A^c)^{1/2},
\end{align*}
where $A^c$ denote the complementary event of $A$ and $C_2$ is a constant that depends only on $\|a\|$, $\|b\|$, and $d$. From \cref{thm:covariance} and the Markov inequality bound, we can upperbound this term by a term of order $\min(n,L)^{-1/2}$ with a constant that depends only on $\|a\|$, $\|b\|$ and $d$. A similar treatment of the second term yields the desired result.
\end{proof}

\section{Additional experiments}
\subsection{Pairplots}

\begin{figure}
    \centering
    \includegraphics{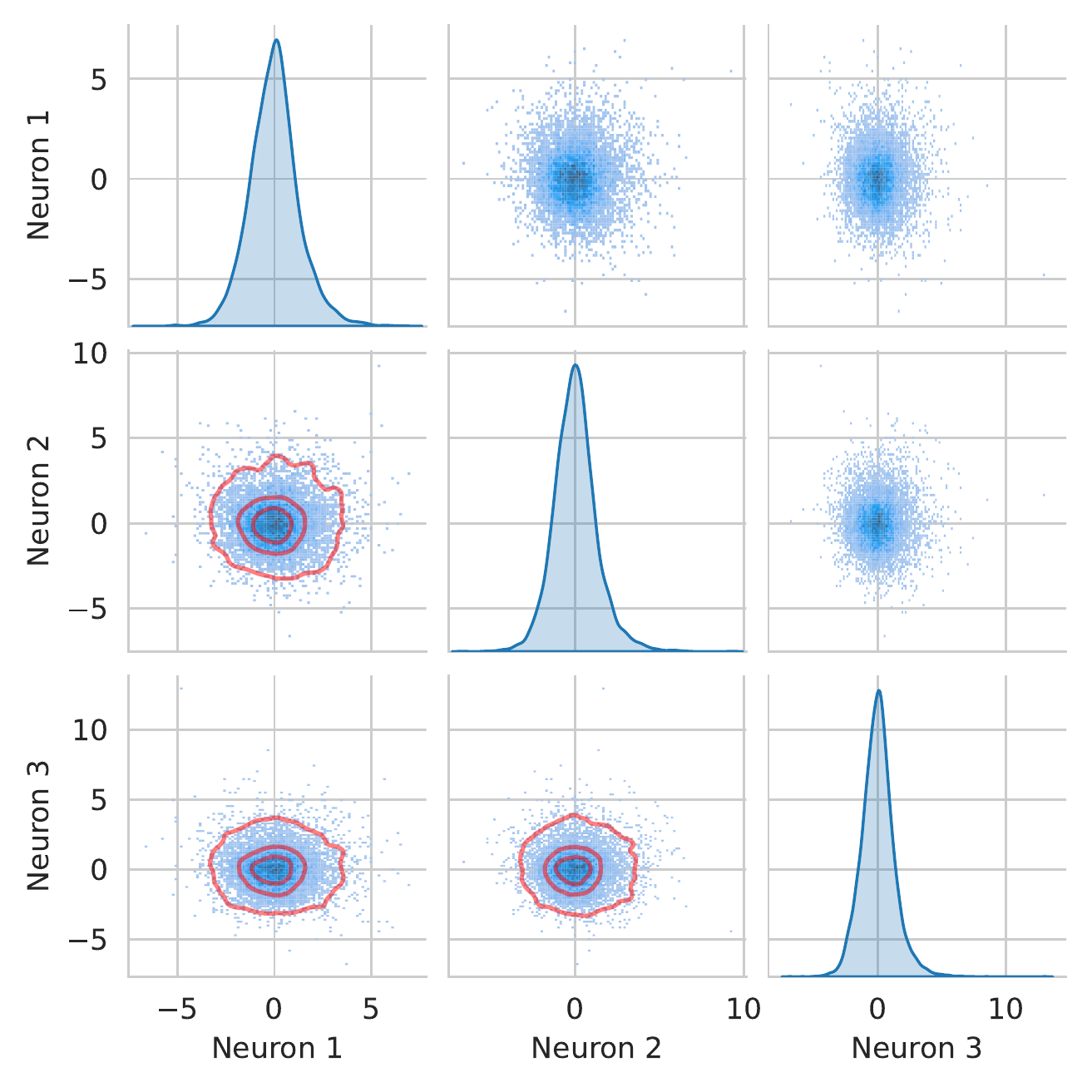}
    \caption{Same plot as \cref{fig:pairplot_500x500} with $n=L=5$}
    \label{fig:my_label}
\end{figure}

\begin{figure}
    \centering
    \includegraphics{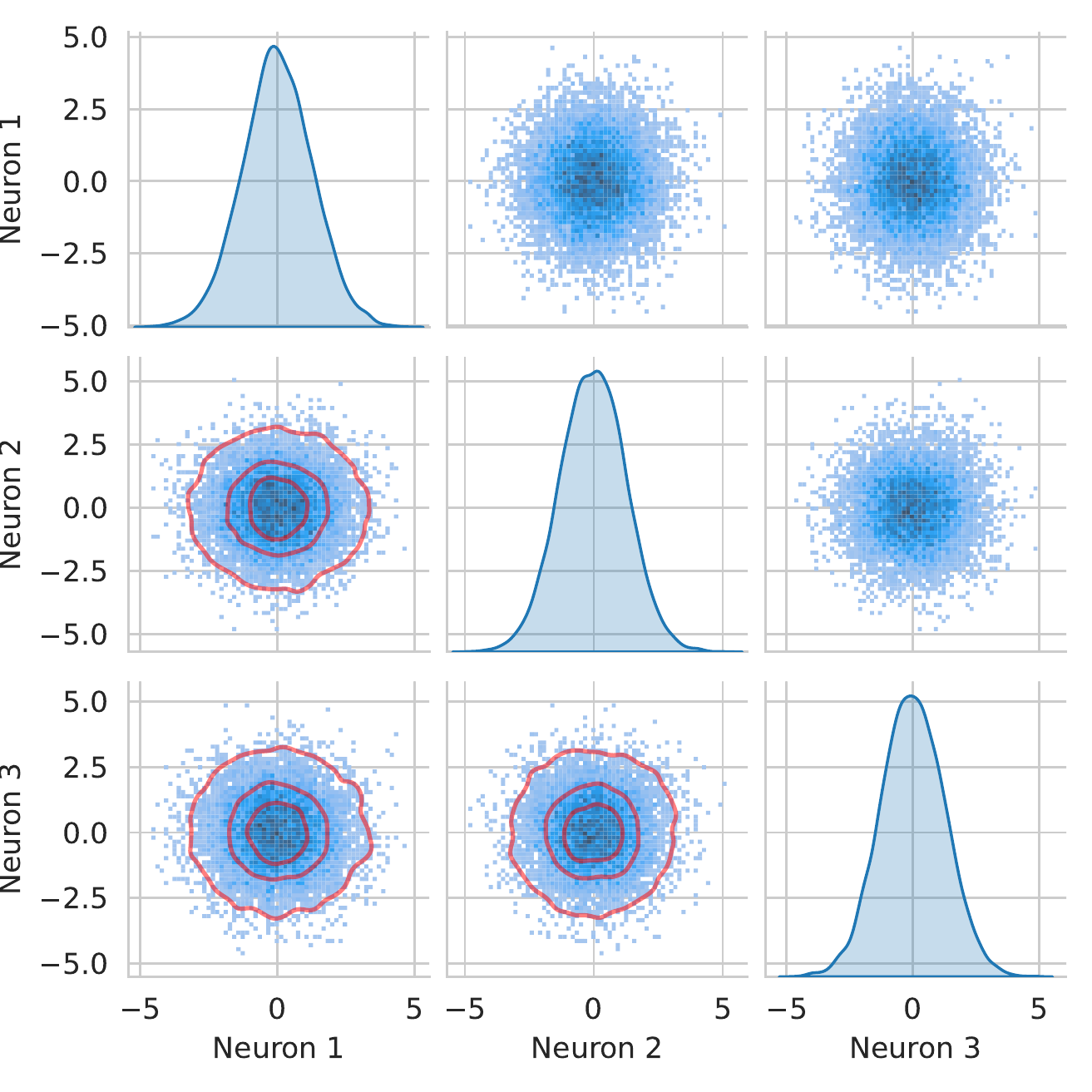}
    \caption{Same plot as \cref{fig:pairplot_500x500} with $n=100$ and $L=5$}
    \label{fig:my_label}
\end{figure}

\begin{figure}
    \centering
    \includegraphics{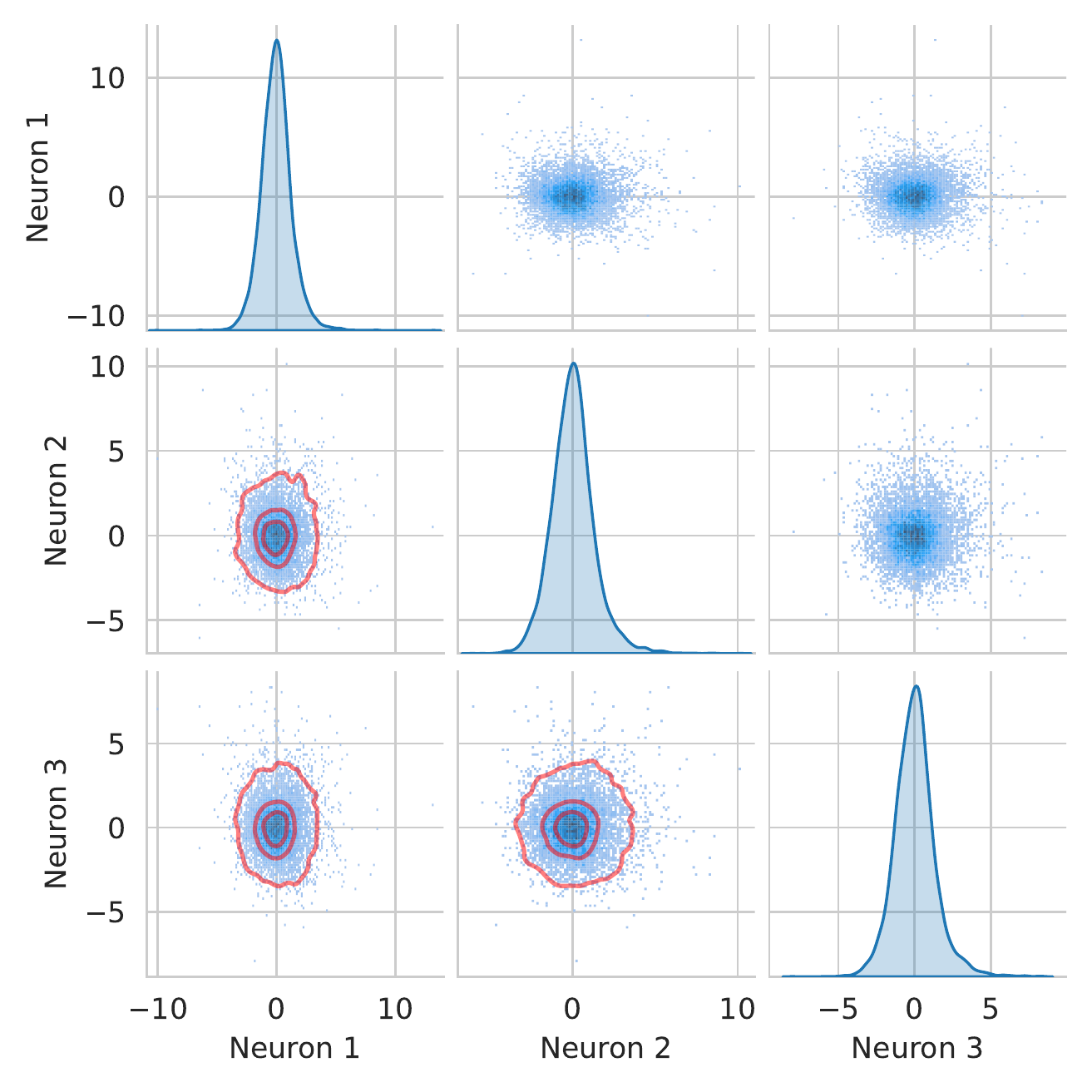}
    \caption{Same plot as \cref{fig:pairplot_500x500} with $n=5$ and $L=100$}
    \label{fig:my_label}
\end{figure}

\end{document}